\documentclass{article}

\usepackage{microtype}
\usepackage{graphicx}
\usepackage{subcaption}

\usepackage{booktabs} 

\usepackage{amssymb}
\usepackage{amsmath}
\usepackage{amsthm}
\usepackage{mathtools}
\usepackage{bbm}
\usepackage{color}

\usepackage{macros}

\usepackage{hyperref}
\newtheorem{theorem}{Theorem}[subsection]
\newtheorem{lemma}[theorem]{Lemma}
\newtheorem{proposition}[theorem]{Proposition}
\newtheorem{condition}[theorem]{Condition}

\usepackage[accepted]{icml2021}

\icmltitlerunning{Fairness of Exposure in Stochastic Bandits}

\begin{document}

\twocolumn[
\icmltitle{Fairness of Exposure in Stochastic Bandits}




\begin{icmlauthorlist}
\icmlauthor{Lequn Wang}{cornell}
\icmlauthor{Yiwei Bai}{cornell}
\icmlauthor{Wen Sun}{cornell}
\icmlauthor{Thorsten Joachims}{cornell}
\end{icmlauthorlist}

\icmlaffiliation{cornell}{Department of Computer Science, Cornell University, Ithaca, NY, USA}

\icmlcorrespondingauthor{Lequn Wang}{lw633@cornell.edu}
\icmlcorrespondingauthor{Wen Sun}{ws455@cornell.edu}
\icmlcorrespondingauthor{Thorsten Joachims}{tj@cs.cornell.edu}


\vskip 0.3in
]



\printAffiliationsAndNotice{}  

\begin{abstract}
Contextual bandit algorithms have become widely used for recommendation in online systems (e.g. marketplaces, music streaming, news), where they now wield substantial influence on which items get exposed to the users. This raises questions of fairness to the items --- and to the sellers, artists, and writers that benefit from this exposure. We argue that the conventional bandit formulation can lead to an undesirable and unfair winner-takes-all allocation of exposure. To remedy this problem, we propose a new bandit objective that guarantees merit-based fairness of exposure to the items while optimizing utility to the users. We formulate fairness regret and reward regret in this setting, and present algorithms for both stochastic multi-armed bandits and stochastic linear bandits. We prove that the algorithms achieve sub-linear fairness regret and reward regret. Beyond the theoretical analysis, we also provide empirical evidence that these algorithms can fairly allocate exposure to different arms effectively. 
\end{abstract}

\section{Introduction}

Bandit algorithms~\cite{thompson1933likelihood,robbins1952some,bubeck2012regret,slivkins2019introduction,lattimore2020bandit} provide an attractive model of learning for online platforms, and they are now widely used to optimize retail, media streaming, and news-feed. Each round of bandit learning corresponds to an interaction with a user, where the algorithm selects an arm (e.g. product, song, article), observes the user's response (e.g. purchase, stream, read), and then updates its policy. Over time, the bandit algorithm thus learns to maximize the user responses, which are often well aligned with the objective of the online platform (e.g. profit maximization, engagement maximization). 

While maximizing user responses may arguably be in the interest of the platform and its users at least in the short term, there is now a growing understanding that it can also be problematic in multiple respects. In this paper, we focus on the fact that this objective ignores the interests of the items (i.e. arms), which also derive utility from the interactions. In particular, sellers, artists and writers have a strong interest in the exposure their items receive, as it affects their chance to get purchased, streamed or read. It is well understood that algorithms that maximize user responses can be unfair in how they allocate exposure to the items~\cite{singh2018fairness}. For example, two items with very similar merit (e.g. click-through rate) may receive substantially different amounts of exposure --- which is not only objectionable in itself, but can also degrade the long-term objectives of the platform (e.g. seller retention~\cite{mehrotra2018towards}, anti-discrimination~\cite{noble2018algorithms}, anti-polarization~\cite{epstein2015search}).
 
To illustrate the problem, consider a conventional (non-personalized) stochastic multi-armed bandit algorithm that is used to promote new music albums on the front-page of a website. The bandit algorithm will quickly learn which album draws the largest click-through rate and keep displaying this album, even if other albums are almost equally good. This promotes a winner-takes-all dynamic that creates superstars \cite{mehrotra2018towards}, and may drive many deserving artists out of business. Analogously, a (personalized) contextual bandit for news-feed recommendation can polarize a user by quickly learning which type of articles the user is most likely to read, and then exclusively recommend such articles instead of a portfolio that is more reflective of the user's true interest distribution.

To overcome these problems of the conventional bandit objective, we propose a new formulation of the bandit problem that implements the principle of Merit-based Fairness of Exposure~\cite{singh2018fairness,BiegaGW18,Wang/etal/21b}. For brevity, we call this the \FairX\ bandit problem. It incorporates the additional fairness requirement that each item/arm receives a share of exposure that is proportional to its merit. We define the merit of an arm as an increasing function of its mean reward, and the exposure as the probability of being selected by the bandit policy at each round. Based on these quantities, we then formulate the reward regret and the fairness regret so that minimizing these two regrets corresponds to maximizing responses while minimizing unfairness to the items. 

For the \FairX\ bandit problem,
we present a fair upper confidence bound (UCB) algorithm and a fair Thompson sampling (TS) algorithm in the stochastic multi-armed bandits (MAB) setting, as well as a fair linear UCB algorithm and a fair linear TS algorithm in the stochastic linear bandits setting. We prove that all algorithms achieve fairness regret and reward regret with sub-linear dependence on the number of rounds, while the TS-based algorithms have computational advantages. The fairness regret of these algorithms also depends on the minimum merit of the arms and a bounded Lipschitz constant of the merit function, and we provide fairness regret lower bounds based on these quantities. Beyond the theoretical analysis, we also conduct an empirical evaluation that compares these algorithms with conventional bandit algorithms and more naive baselines, finding that the fairness-aware algorithms can fairly allocate exposure to different arms effectively while maximizing user responses. 
\section{Related Work}

The bandit problem was first introduced by Thompson~\cite{thompson1933likelihood} to efficiently conduct medical trials. Since then, it has been extensively studied in different variants, and we refer to these books~\cite{bubeck2012regret,slivkins2019introduction,lattimore2020bandit} for a comprehensive survey. We focus on the classic stochastic MAB setting where each arm has a fixed but unknown reward distribution, as well as the stochastic linear bandits problem where each arm is represented as a $\Dimension$-dimensional vector and its expected reward is a linear function of its vector representation. In both stochastic MAB and stochastic linear bandits, some of the algorithms we designed leverage the idea of optimism in the face of uncertainty behind the UCB algorithm~\cite{lai1985asymptotically}, while other algorithms leverage the idea of posterior sampling behind the TS~\cite{thompson1933likelihood} algorithm. The theoretical results of the proposed fair UCB and fair linear UCB algorithm borrow some ideas from several prior finite time analysis works on the conventional UCB and linear UCB algorithm~\cite{auer2002using, dani2008stochastic, abbasi2011improved}. We adopt the Bayesian regret framework~\cite{russo2014learning} for our theoretical analysis of the fair TS and the fair linear TS algorithm. 

Algorithmic fairness has been extensively studied in binary classification~\cite{HardtPNS16,chouldechova2017fair,KleinbergMR17,agarwal2018reductions}. These works propose statistical criteria to test algorithmic fairness that often operationalize definitions of fairness from political philosophy and sociology. Several prior works~\cite{blum2018preserving,blum2019advancing,bechavod2019equal} study how to achieve these fairness criteria in online learning. These algorithms achieve fairness to the incoming users. We, in contrast, achieve fairness to the arms.  

Joseph et al. \cite{NIPS2016_6355,joseph2016fair,joseph2018meritocratic} study fairness in bandits that ensure a better arm is always selected with no less probability than a worse arm. Different from our definition of fairness, their optimal policy is still the one that deterministically selects the arm with the largest expected reward while giving zero exposure to all the other arms. Another type of fairness definition in bandits is to ensure a minimum and/or maximum amount of exposure to each arm or group of arms~\cite{heidari2018preventing,wen2019fairness,schumann2019group,li2019combinatorial,celis2018algorithmic,claure2020multi,patil2020achieving,chen2020fair}. However, they do not take the merit of the items into consideration. Gillen et al.~\cite{gillen2018online} propose to optimize individual fairness defined in \cite{dwork2012fairness} in the adversarial linear bandits setting, where the difference between the probabilities that any two arms are selected is bounded by the distance between their context vectors. They require additional feedback of fairness-constraint violations. We work in the stochastic bandits setting and we do not require any additional feedback beyond the reward. We also ensure that similar items obtain similar exposure, but we focus on similarity of merit, which corresponds to closeness in mean reward conditioned on context. 

The most relevant work may arguably be~\cite{liu2017calibrated}, which considers fairness in stochastic MAB problems where the reward distribution is Bernoulli. They aim to achieve calibrated fairness where each arm is selected with the probability equal to that of its reward being the largest, while satisfying a smoothness constraint where arms with similar merit should receive similar exposure. They propose a TS-based algorithm that achieves fairness regret with a  $T^{2/3}$ dependence on the time horizon $T$.  Our formulation is more general in a sense that we consider arbitrary reward distributions and merit functions, with their formulation as a special case. What is more, our proposed algorithms achieve fairness regret with a $\sqrt{T}$ dependence on $T$. In addition, we further study the more general setting of stochastic linear bandits. 

Our definition of fairness has connections to the fair division problem~\cite{steihaus1948problem,brams1996fair,procaccia2013cake}, where the goal is to allocate a resource to different agents in a fair way. In our problem, we aim to allocate the users' attention among the items in a fair way. Our definition of fairness ensures proportionality, one of the key desiderata in the fair division literature to ensure each agent receives its fair share of the resource. Recently, merit-based fairness of exposure has been studied for ranking problems in the statistical batch learning framework~\cite{singh2018fairness,singh2019policy}. We build upon this work, and extend merit-based fairness of exposure to the online-learning setting.

\section{Stochastic Multi-Armed Bandits in the \FairX\ Setting}

We begin by introducing the \FairX\ setting for stochastic MAB, including our new formulation of fairness and reward regret. We then develop two algorithms, called \FairXUCB\ and \FairXTS, and bound their fairness and reward regret. In the subsequent section, we will extend this approach to stochastic linear bandits.

\subsection{\FairX\ Setting for Stochastic MAB}

A stochastic MAB instance can be represented as a collection of reward distributions $\BanditEnv = (\RewardDist_\Action:\Action\in[\NumActions])$, where $\RewardDist_\Action$ is the reward distribution of arm $\Action$ with mean $\Params^\star_\Action = \mE_{\Reward\sim\RewardDist_\Action}\left[\Reward\right]$. The learner interacts with the environment sequentially over $\TimeSteps$ rounds. In each round $t\in[\TimeSteps]$, the learner has to choose a policy $\Policy_t$ over the $\NumActions$ arms based on the interaction history before round $t$. The learner then samples an arm $\Action_t\sim\Policy_t$. In response to the selected arm $\Action_t$, the environment samples a reward $\Reward_{t,\Action_t}\sim\RewardDist_{
\Action_t}\in\mathbb{R}$ from the reward distribution $\RewardDist_{\Action_t}$ and reveals the reward $\Reward_{t,\Action_t}$ to the learner. The history $\History_t = \left(\Policy_1, \Action_1,\Reward_{1,\Action_1},\ldots, \Policy_{t-1}, \Action_{t-1},\Reward_{t-1,\Action_{t-1}} \right)$ consists of all the deployed policies, chosen arms, and their associated rewards.  Conventionally, the goal of learning is to maximize the cumulative expected reward $ \sum_{t=1}^{\TimeSteps}\mE_{\Action_t\sim\Policy_t}\Params^\star_{\Action_t}$. Thus conventional bandit algorithms converge to a policy that deterministically selects the arm with the largest expected reward. 

As many have pointed out in other contexts \cite{singh2018fairness,mehrotra2018towards,BiegaGW18,beutel2019fairness,geyik2019fairness,abdollahpouri2020multistakeholder}, such winner-takes-all allocations can be considered unfair to the items in many applications and can lead to undesirable long-term dynamics. Bringing this insight to the task of bandit learning, we propose to incorporate merit-based fairness-of-exposure constraints~\cite{singh2018fairness} into the bandits objective. Specifically, we aim to learn a policy $\pi^\star$ which ensures that each arm receives an amount of exposure proportional to its merit, where merit is quantified through an application-dependent merit function $\Merit(\cdot)>0$ that maps the expected reward of an arm to a positive merit value. 
\[
     \frac{\Policy^\star(\Action)}{\Merit(\Params^\star_\Action)} = \frac{\Policy^\star(\Action')}{\Merit(\Params^\star_{\Action'})}\quad \forall \Action, \Action' \in [\NumActions].
\]
The merit function $\Merit$ is an input to the bandit algorithm, and it provides a design choice that permits tailoring the fairness criterion to different applications. The following theorem shows that there is a unique policy that satisfies the above fairness constraints. 

\begin{theorem}[Optimal Fair Policy]
\label{theo:unique_optimal_fair_policy}
For any mean reward parameter $\Params^\star$ and any choice of merit function $\Merit(\cdot)>0$, there exist a unique policy $\Policy^\star$ of the form
\[
     \Policy^{\star}(\Action) = \frac{\Merit(\Params^\star_\Action)}{\sum_{\Action'}\Merit(\Params^\star_{\Action'}) }\quad\forall \Action\in[\NumActions], 
\]
that fulfills the merit-based fairness of exposure constraints. 
\end{theorem}

We refer to $\Policy^\star$ as the optimal fair policy. All the proofs of the theorems are in Appendix~\ref{section:proof_theorems}. 

When the bandit converges to this optimal fair policy $\Policy^\star$, the expected reward also converges to the expected reward of the optimal fair policy. We thus define the {\em reward regret} $\RewardRegret_\TimeSteps$ at round $\TimeSteps$ as the gap between the expected reward of the deployed policy and the expected reward of the optimal fair policy $\Policy^\star$
 \begin{equation}
 \begin{split}
    \RewardRegret_\TimeSteps &=\sum_{t=1}^\TimeSteps\sum_{\Action}\Policy^\star(\Action)\Params^\star_\Action - \sum_{t=1}^\TimeSteps\sum_{\Action}\Policy_t(\Action)\Params^\star_\Action.
 \end{split}
 \end{equation}
While this reward regret quantifies how quickly the reward is optimized, we also need to quantify how effectively the algorithm learns to enforce fairness. We thus define the following {\em fairness regret} $\FairnessRegret_\TimeSteps$, which measures the cumulative $\ell^1$ distance between the deployed policy and the optimal fair policy at round $\TimeSteps$
\begin{equation}
\FairnessRegret_\TimeSteps = \sum_{t=1}^{\TimeSteps}\sum_\Action \lvert\Policy^{\star}(\Action) - \Policy_t(\Action)\rvert.
\end{equation}
The fairness regret and the reward regret depend on both the randomly sampled rewards, as well as the arms randomly sampled from the policy. They are thus random variables and we aim to minimize the regrets with high probability. 

To prepare for the theoretical analysis, we introduce the following two conditions on the merit function $\Merit$ to suitably characterize a \FairX\ bandit problem. 

\begin{condition}[Minimum Merit] \label{condi:min_merit}
The merit of each arm is positive, i.e. $\min_{\Params}\Merit(\Params)\geq\MeritMin$ for some positive constant $\MeritMin>0$.   
\end{condition}

\begin{condition}[Lipschitz Continuity] \label{condi:lip_cont}
The merit function $\Merit$ is $\LipConst$-Lipschitz continuous, i.e. $\forall$ $\Params_1,\Params_2$, $\lvert\Merit(\Params_1)-\Merit(\Params_2)\rvert\leq\LipConst\lvert\Params_1-\Params_2\rvert$ for some positive constant $\LipConst>0$. 
\end{condition}

The following two theorems show that neither of the two conditions can be dropped, if we want to obtain bandit algorithms with fairness regret that is sub-linear in the number of rounds $\TimeSteps$. 

\begin{theorem}[Lower Bound on Fairness Regret is Linear without Minimum-Merit Condition]
\label{theo:lower_bound_min_merit}
For time horizon $\TimeSteps>0$, there exists a $1$-Lipschitz continuous merit function $f$ where $\min_{\Params} \Merit(\Params) = 1/\sqrt{\TimeSteps}$, such that for any bandit algorithm, there must exist a MAB instance such that the expected fairness regret is at least $\mE\left[\FairnessRegret_\TimeSteps\right] \geq 0.015 \TimeSteps$.
\end{theorem}

\begin{theorem}[Lower Bound on Fairness Regret is Linear without Bounded Lipschitz-Continuity Condition]
\label{theo:lower_bound_lip_cont}
For time horizon $\TimeSteps>0$, there exists a $\sqrt{\TimeSteps}$-Lipschitz continuous merit function $f$ with minimum merit $1$, such that for any bandit algorithm, there must exist a MAB instance such that the expected fairness regret is at least $\mE[\FairnessRegret_\TimeSteps] \geq 0.015 \TimeSteps$.
\end{theorem}

\subsection{\FairXUCB\ Algorithm}

\begin{algorithm}[tbh]
\begin{algorithmic}[1]
\STATE {\bf input: }{$\NumActions$, $\TimeSteps$,  $\Merit$, $\ConfiWidth$}
\FOR{$t=1$ to $\TimeSteps$}
\STATE $\forall\Action$ $\SelectedTimes_{t,\Action} = \sum_{\tau=1}^{t-1}\Indicator\{\Action_\tau=\Action\}$ 
\STATE $\forall\Action$ $\hat{\Params}_{t,\Action} =\sum_{\tau=1}^{t-1}\Indicator\{\Action_\tau=\Action\}\Reward_{\tau,\Action_\tau}/\SelectedTimes_{t,\Action}$
\STATE $\forall\Action$ $w_{t,\Action} = \ConfiWidth/\sqrt{\SelectedTimes_{t,\Action}}$
\STATE $\ConfidenceRegion_t = \left(\Params:\forall \Action\:\Params_\Action\in\left[ \hat{\Params}_{t,\Action}-w_{t,\Action}, \hat{\Params}_{t,\Action}+ w_{t,\Action} \right]\right)$
\STATE \label{alg:fair_UCB:optimization} $\Params_t = \argmax_{\Params\in\ConfidenceRegion_t}\sum_\Action\frac{ \Merit(\Params_\Action)}{\sum_{\Action'}\Merit(\Params_{\Action'})}\Params_\Action$ 
\STATE Construct policy $\Policy_t(\Action) = \frac{\Merit(\Params_{t,\Action})}{\sum_{\Action'} \Merit(\Params_{t,\Action'})}$ 
\STATE Sample arm $\Action_t\sim\Policy_t$
\STATE Observe reward $\Reward_{t,\Action_t}$
\ENDFOR
\end{algorithmic}
\caption{\FairXUCB\ Algorithm}
\label{alg:fair_UCB}

\end{algorithm}

The first algorithm we introduce is called \FairXUCB\ and it is detailed in Algorithm~\ref{alg:fair_UCB}. It utilizes the idea of optimism in the face of uncertainty. At each round $t$, the algorithm constructs a confidence region $\ConfidenceRegion_t$ which contains the true parameter $\Params^\star$ with high probability. Then the algorithm optimistically selects a parameter $\Params_t \in \mathbb{R}^{K}$ within the confidence region $\ConfidenceRegion_t$ that maximizes the estimated expected reward subject to the constraint that we construct a fair policy as if the selected parameter is the true parameter. Compared to the conventional UCB algorithm which selects the arm with the largest UCB deterministicly in each round, the proposed \FairXUCB\ algorithm selects arms stochastically to ensure fairness. Finally, we apply the constructed policy $\Policy_t$, observe the feedback, and update the confidence region. 

The following two theorems characterize the fairness and reward regret upper bounds of the \FairXUCB\ algorithm. 
\begin{theorem}[\FairXUCB\  Fairness Regret]
\label{theo:fair_UCB_FR}
Under Condition ~\ref{condi:min_merit} and ~\ref{condi:lip_cont}, suppose $\forall t,\Action: \Reward_{t,\Action}\in[-1,1]$, when $\TimeSteps>\NumActions$, for any $\delta\in(0,1)$, set $\ConfiWidth = \sqrt{2\ln\left(4\TimeSteps\NumActions/\delta\right)}$,  the fairness regret of the \FairXUCB\ algorithm is $\FairnessRegret_\TimeSteps =\widetilde{O}\left( \LipConst\sqrt{\NumActions\TimeSteps}/\MeritMin\right)$ with probability at least $1-\delta$. 
\end{theorem}

\begin{theorem}[\FairXUCB\ Reward Regret]
\label{theo:fair_UCB_RR}
Suppose $\forall t,\Action: \Reward_{t,\Action}\in[-1,1]$, when $\TimeSteps>\NumActions$, for any $\delta\in(0,1)$, set $\ConfiWidth = \sqrt{2\ln\left(4\TimeSteps\NumActions/\delta\right)}$, the reward regret of the \FairXUCB\ algorithm is $\RewardRegret_\TimeSteps =\widetilde{O}\left( \sqrt{\NumActions\TimeSteps}\right)$ with probability at least $1-\delta$.  
\end{theorem}

$\widetilde{O}$ ignores logarithmic factors in $O$. Note that the well-known $\Omega\left(\sqrt{\NumActions\TimeSteps}\right)$ reward regret lower bound~\cite{auer2002nonstochastic} developed for the conventional bandit problem also holds for the \FairX\ setting because the conventional stochastic MAB problem that only minimizes the reward regret is a special case of the \FairX\ setting where we set the merit function $\Merit$ to be an infinitely steep increasing function. Since the reward regret upper bound of \FairXUCB\ we proved does not depend on Conditions~\ref{condi:min_merit} and ~\ref{condi:lip_cont} about the merit function $\Merit$, our reward regret upper bound of the \FairXUCB\ algorithm is tight up to logarithmic factors. 

The fairness regret has the same dependence on the number of arms $\NumActions$ and the number of rounds $\TimeSteps$ as the reward regret. It further depends on the minimum merit constant $\MeritMin$ and the Lipschitz continuity constant $\LipConst$, which we treat as absolute constants due to Theorem~\ref{theo:lower_bound_min_merit} and Theorem~\ref{theo:lower_bound_lip_cont}.  

Compared to Fair\_SD\_TS algorithms proposed in~\cite{liu2017calibrated}, our proposed \FairXUCB\ algorithm focuses on fairness and reward regret across rounds instead of achieving a smooth fairness constraint for each round. This allows \FairXUCB\ to achieve improved fairness and reward regret ($\sqrt{\NumActions\TimeSteps}$ compared to $\left(\NumActions\TimeSteps\right)^{2/3}$). In addition, \FairXUCB\ works for general reward distributions and merit functions while SD\_TS only works for Bernoulli reward distribution and identity merit function. 

One challenge in implementing Algorithm~\ref{alg:fair_UCB} lies in Step~\ref{alg:fair_UCB:optimization}, since finding the most optimistic parameter is a non-convex constrained optimization problem. We solve this optimization problem approximately with projected gradient descent in our empirical evaluation. In the next subsection, we will introduce the \FairXTS\ algorithm that avoids this optimization problem. 

\subsection{\FairXTS\ Algorithm}
\begin{algorithm}[t]
\begin{algorithmic}[1]
\STATE{\bf input: }{$\Merit$, $\mathcal{V}_1$}
\FOR{$t=1$ to $\infty$}
\STATE Sample parameter from posterior $\Params_t\sim \mathcal{V}_t$
\STATE Construct policy $\Policy_t(\Action)=\frac{\Merit(\Params_{t,\Action})}{\sum_{\Action'}\Merit(\Params_{t,\Action'})}$
\STATE Sample arm $\Action_t\sim\Policy_t$
\STATE Observe reward $\Reward_{t,\Action_t}$
\STATE Update posterior $\mathcal{V}_{t+1}=\text{Update}(\mathcal{V}_1, \History_{t+1})$
\ENDFOR
\end{algorithmic}
\caption{\FairXTS\ Algorithm}
\label{alg:fair_TS}
\end{algorithm}

Another approach to designing stochastic bandit algorithms that has proven successful both empirically and theoretically is Thompson Sampling (TS). We find that this approach can also be applied to the \FairX\ setting. In particular, our \FairXTS\ as shown in Algorithm~\ref{alg:fair_TS} uses posterior sampling similar to a conventional TS bandit. The algorithm puts a prior distribution $\mathcal{V}_1$ on the expected reward of each arm $\Params^\star$. For each round $t$, the algorithm samples a parameter $\Params_t$ from the posterior $\mathcal{V}_t$, and constructs a fair policy $\Policy_t$ from the sampled parameter to deploy. Finally, the algorithm observes the feedback and updates the posterior distribution of the true parameter. 

Following~\cite{russo2014learning}, we analyze the Bayesian reward and fairness regret of the algorithm. The Bayesian regret framework assumes that the true parameter $\Params^\star$ is sampled from the prior, and the Bayesian regret is the expected regret taken over the prior distribution
\begin{equation}
    \text{Bayes}\RewardRegret_\TimeSteps = \mE_{\Params^\star}\left[\mE[\RewardRegret_\TimeSteps\vert\Params^\star]\right]
\end{equation}
\begin{equation}
    \text{Bayes}\FairnessRegret_\TimeSteps = \mE_{\Params^\star}\left[\mE[\FairnessRegret_\TimeSteps\vert\Params^\star]\right]. 
\end{equation}
In the following two theorems we provide bounds on both the Bayesian reward regret and the Bayesian fairness regret of the \FairXTS\ algorithm. 

\begin{theorem}[\FairXTS\ Fairness Regret]
\label{theo:fair_TS_FR}
Under Condition~\ref{condi:min_merit} and ~\ref{condi:lip_cont}, 
suppose the mean reward $\Params^\star_a$ of each arm $a$ is independently sampled from standard normal distribution $\Normal(0,1)$, and $\forall t,\Action$  $\Reward_{t,\Action}\sim\Normal(\Params^\star_\Action,1)$, the Bayesian fairness regret of the \FairXTS\ algorithm at any round $\TimeSteps$ is $\text{Bayes}\FairnessRegret_\TimeSteps=\widetilde{O}\left( \LipConst\sqrt{\NumActions\TimeSteps}/\MeritMin \right)$. 
\end{theorem}

\begin{theorem}[\FairXTS\ Reward Regret]
\label{theo:fair_TS_RR}
Suppose the mean reward $\Params^\star_a$ of each arm $a$ is independently sampled from standard normal distribution $\Normal(0,1)$, and $\forall t,\Action$  $\Reward_{t,\Action}\sim\Normal(\Params^\star_\Action,1)$,
 the Bayesian fairness regret of the \FairXTS\ algorithm at any round $\TimeSteps$ is $\text{Bayes}\RewardRegret_\TimeSteps=\widetilde{O}\left( \sqrt{\NumActions\TimeSteps} \right)$. 
\end{theorem}

Note that these regret bounds are on the same order as the fairness and reward regret of the \FairXUCB\ algorithm. However, \FairXTS\ relies on sampling from the posterior and thus avoids the non-convex optimization problem that makes the use of \FairXUCB\ more challenging.

\section{Stochastic Linear Bandits in the \FairX\ Setting}

In this section, we extend the two algorithms introduced in the MAB setting to the more general stochastic linear bandits setting where the learner is provided with contextual information for making decisions. We discuss how the two algorithms can be adapted to this setting to achieve both sub-linear fairness and reward regret.

\subsection{\FairX\ Setting for Stochastic Linear Bandits}

In stochastic linear bandits, each arm $\Action$ at round $t$ comes with a context vector $\Context_{t,\Action}\in\mathbb{R}^\Dimension$. A stochastic linear bandits instance $\BanditEnv=(\RewardDist_{\Context}:\Context\in\mathbb{R}^\Dimension)$ is a collection of reward distributions for each context vector. The key assumption of stochastic linear bandits is that there exists a true parameter $\Params^\star$ such that, regardless of the interaction history $\History_t$, the mean reward of arm $\Action$ at round $t$ is the product between the context vector and the true parameter $\mE_{\Reward\sim\RewardDist_{\Context_{t,\Action}}}[\Reward\vert\History_t] =\Params^\star\cdot\Context_{t,\Action}$ for all $t, \Action$. The noise sequence
\[
\eta_t = \Reward_{t,\Action_t} - \Params^{\star}\cdot\Context_{t,\Action_t}
\]
is thus a martingale difference sequence, since
\[
\mE[\eta_t|\History_t] = \mE_{\Action\sim\Policy_t}[ \mE_{\Reward\sim\RewardDist_{\Context_{t,\Action}}}[\Reward|\History_t] - \Params^{\star}\cdot\Context_{t,\Action}]=0. 
\]
At each round $t$, the learner is given a set of context vectors $\Actions_t\subset\mathbb{R}^\Dimension$ representing the arms, and it has to choose a policy $\Policy_t$ over these $\NumActions$ arms based on the interaction history $\History_t = (\Actions_1,\Policy_1, \Action_1,\Reward_{1,\Action_1},\ldots, \Actions_{t-1}, \Policy_{t-1}, \Action_{t-1},\Reward_{t-1,\Action_{t-1}})$. We focus on problems where the number of available arms is finite  $\forall t: \lvert\Actions_t\rvert=\NumActions$, but where $K$ could be large. 

 Again, we want to ensure that the policy provides each arm with an amount of exposure proportional to its merit
\[
     \frac{\Policy^\star_t(\Action)}{\Merit(\Params^\star\cdot\Context_{t,\Action})} = \frac{\Policy^\star_{t}(\Action')}{\Merit(\Params^\star\cdot\Context_{t,\Action'})}\quad \forall t, \Context_{t,\Action}, \Context_{t,\Action'} \in \Actions_t,
\]
where $\Merit$ is the merit function that maps the mean reward of the arm to a positive merit value. Since the set of arms changes over time, the optimal fair policy $\Policy_t^\star$ at round $t$ is time-dependent
\[
    \Policy^\star_t(\Action) = \frac{\Merit(\Params^{\star}\cdot \Context_{t,\Action})}{\sum_{\Action'}\Merit(\Params^{\star}\cdot \Context_{t,\Action'})}\quad\forall t, \Action.
\]
Analogous to the MAB setting, we define the reward regret as the expected reward difference between the optimal fair policy and the deployed policy
\begin{equation}
    \RewardRegret_{\TimeSteps} = \sum_{t=1}^{\TimeSteps}\sum_\Action \Policy^\star_t(\Action)\Params^{\star}\cdot\Context_{t,\Action} - \sum_{t=1}^{\TimeSteps}\sum_\Action \Policy_t(\Action)\Params^{\star}\cdot\Context_{t,\Action},
\end{equation}
and fairness regret as the cumulative $\ell^1$ distance between the optimal fair policy and the deployed policy
\begin{equation}
    \FairnessRegret_\TimeSteps = \sum_{t=1}^{\TimeSteps}\sum_\Action \lvert\Policy^{\star}_t(\Action) - \Policy_t(\Action)\rvert.
\end{equation}
The lower bounds on the fairness regret derived in Theorem~\ref{theo:lower_bound_min_merit} and Theorem~\ref{theo:lower_bound_lip_cont} in the MAB setting also apply to the stochastic linear bandit setting, since we can easily convert a MAB instance into a stochastic linear bandits instance by constructing $\NumActions$ $\NumActions$-dimensional basis vectors, each representing one arm. Thus we again employ Condition~\ref{condi:min_merit} and~\ref{condi:lip_cont} to design algorithms that have fairness regret with sub-linear dependence on the horizon $\TimeSteps$.

\subsection{\FairXLinUCB\ Algorithm}

\begin{algorithm}[t]
\begin{algorithmic}[1]
\STATE{\bf input: } { $\beta_t$, $\Merit$, $\lambda$}
\STATE{\bf initialization:} {
$\CovMatrix_1 = \lambda\mathbf{I}_\Dimension$, $\B_1 = \mathbf{0}_\Dimension$
}
\FOR { $t=1$ to $\infty$ }
\STATE Observe contexts $\Actions_t =$ $(\Context_{t,1},$ $ \Context_{t,2},$ $\ldots,$ $\Context_{t,\NumActions})$
\STATE $\hat{\Params}_t = \CovMatrix_t^{-1}\B_t$ \{The ridge regression solution\}
\STATE $\ConfidenceRegion_t=(\Params: \|\Params-\hat{\Params}_t\|_{\CovMatrix_t}\leq \sqrt{\beta_t} )$ 
\STATE \label{alg:fair_LinUCB_optimization}$\Params_t = \argmax_{\Params\in \ConfidenceRegion_t} \sum_{\Action}\frac{\Merit(\Params\cdot\Context_{t,\Action})}{ \sum_{\Action'} \Merit(\Params\cdot\Context_{t,\Action'}) }\Params\cdot\Context_{t,\Action}$
\STATE Construct policy $\Policy_t(\Action)=\frac{\Merit(\Params_t\cdot\Context_{t,\Action})}{\sum_{\Action'}\Merit(\Params_t\cdot\Context_{t,\Action'})}$
\STATE Sample arm $\Action_t\sim\Policy_t$
\STATE Observe reward $\Reward_{t,\Action_t}$
\STATE $\CovMatrix_{t+1} =  \CovMatrix_t + \Context_{t,\Action_t}\Context_{t,\Action_t}^\top$
\STATE $\B_{t+1} = \B_t +  \Context_{t,\Action_t} r_{t,\Action_t}$
\ENDFOR
\end{algorithmic}
\caption{\FairXLinUCB\ Algorithm}
\label{alg:fair_LinUCB}
\end{algorithm}

Similar to the \FairXUCB\ algorithm, the \FairXLinUCB\ algorithm constructs a confidence region $\ConfidenceRegion_{t}$ of the true parameter $\Params^\star$ at each round $t$. The center of the confidence region $\hat{\Params}_t$ is the solution of a ridge regression over the existing data, which can be updated incrementally. The radius of the confidence ball $\beta_t$ is an input to the algorithm. The algorithm proceeds by repeatedly selecting a parameter $\Params_t$ that is optimistic about the expected reward within the confidence region, subject to the constraint that we construct a fair policy from the parameter.  We prove the following upper bounds on the fairness regret and reward regret of the \FairXLinUCB\ algorithm. 

\begin{theorem} [\FairXLinUCB\ Fairness Regret]
\label{theo:fair_LinUCB_FR}
Under Condition~\ref{condi:min_merit} and~\ref{condi:lip_cont}, suppose $\forall t,a$ $\|\Context_{t,\Action}\|_2\leq 1$, $\eta_t$ is $1$ sub-Gaussian, $\|\Params^{\star}\|_2\leq 1$, set $\lambda=1$, with proper choice of $\beta_t$, the fairness regret at any round $\TimeSteps>0$ is $\FairnessRegret_\TimeSteps=\widetilde{O}\left(\LipConst\Dimension\sqrt{\TimeSteps}/{\MeritMin}\right)$ with high probability. 
\end{theorem}

\begin{theorem}[\FairXLinUCB\ Reward Regret]
\label{theo:fair_LinUCB_RR}
Suppose $\forall t,a$ $\|\Context_{t,\Action}\|_2\leq 1$, $\eta_t$ is $1$ sub-Gaussian, $\|\Params^{\star}\|_2\leq 1$, set $\lambda=1$, with proper choice of $\beta_t$, the reward regret at any round $\TimeSteps>0$ is $\RewardRegret_\TimeSteps=\widetilde{O}\left(\Dimension\sqrt{\TimeSteps}\right)$ with high probability.
\end{theorem}

Both fairness and reward regret have square root dependence on the horizon $\TimeSteps$ and a linear dependence on the feature dimension $\Dimension$, and the fairness regret depends on the absolute constants $\LipConst$ and $\MeritMin$. Note that the reward regret is not tight in terms of $\Dimension$ and there exist algorithms~\cite{chu2011contextual,lattimore2020bandit} that achieve reward regret $\widetilde{O}(\sqrt{\Dimension\TimeSteps})$. However these algorithms are based on the idea of arm elimination and thus will likely not achieve low fairness regret.  Also LinUCB is a much more practical option than the ones based on arm elimination \citep{chu2011contextual}.

The optimization Step~\ref{alg:fair_LinUCB_optimization} in Algorithm~\ref{alg:fair_LinUCB}, where we need to find an optimistic parameter $\Params_t$ that maximizes the estimated expected reward within the confidence region $\ConfidenceRegion_t$ subject to the fairness constraint, is again a non-convex constrained optimization problem. We use projected gradient descent to find approximate solutions in our empirical evaluation.

\subsection{\FairXLinTS\ Algorithm}

\begin{algorithm}[t]
\begin{algorithmic}[1]
\STATE{\bf input: }{ $\Merit$, $\mathcal{V}_1$}
\FOR{$t=1$ to $\infty$}
\STATE Observe contexts $\Actions_t = (\Context_{t,1}, \Context_{t,2}, \ldots, \Context_{t,\NumActions})$
\STATE Sample parameter from posterior $\Params_t\sim \mathcal{V}_t$
\STATE Construct policy $\Policy_t(\Action)=\frac{\Merit(\Params_{t,\Action}\cdot\Context_{t,\Action})}{\sum_{\Action'}\Merit(\Params_{t,\Action'}\cdot\Context_{t,\Action})}$
\STATE Sample arm $\Action_t\sim\Policy_t$
\STATE Observe reward $\Reward_{t,\Action_t}$
\STATE Update posterior $\mathcal{V}_{t+1} = \text{Update}(\mathcal{V}_1, \History_{t+1})$
\ENDFOR
\end{algorithmic}
\caption{\FairXLinTS~Algorithm}
\label{alg:fair_LinTS}
\end{algorithm}
To avoid the difficult optimization problem of \FairXLinUCB, we again explore the use of Thompson sampling. Algorithm~\ref{alg:fair_LinTS} shows our proposed \FairXLinTS. At each round $t$, the algorithm samples a parameter $\Params_t$ from the posterior distribution $\mathcal{V}_t$ of the true parameter $\Params^\star$ and derives a fair policy $\Policy_t$ from the sampled parameter. Then the algorithm deploys the policy and observes the feedback for the selected arm. Finally, the algorithm updates the posterior distribution of the true parameter given the observed data. Note that sampling from the posterior is efficient for a variety of models (e.g. normal distribution), as opposed to the non-convex optimization problem in \FairXLinUCB.

Appropriately extending our definition of Bayesian reward regret and fairness regret
\begin{equation}
    \text{Bayes}\RewardRegret_\TimeSteps = \mE_{\Params^\star}\left[\mE[\RewardRegret_\TimeSteps\vert\Params^\star]\right]
\end{equation}
\begin{equation}
    \text{Bayes}\FairnessRegret_\TimeSteps = \mE_{\Params^\star}\left[\mE[\FairnessRegret_\TimeSteps\vert\Params^\star]\right] ,
\end{equation}
we can prove the following regret bounds for the \FairXLinTS\ algorithm. 

\begin{theorem}[\FairXLinTS\ Fairness Regret]
\label{theo:fair_LinTS_FR}

Under Condition~\ref{condi:min_merit} and~\ref{condi:lip_cont}, suppose each dimension of the true parameter $\Params^\star$ is independently sampled from standard normal distribution $\Normal(0,1)$, $\forall t,a$ $\|\Context_{t,\Action}\|_2\leq 1$, $\eta_t$ is sampled from standard normal distribution $\Normal(0,1)$, the Bayesian fairness regret of the \FairXLinTS\ algorithm is $\text{Bayes}\FairnessRegret=\widetilde{O}\left(\LipConst\sqrt{\Dimension\TimeSteps}/\MeritMin\right)$. 
\end{theorem}

\begin{theorem}[\FairXLinTS\ Reward Regret]
\label{theo:fair_LinTS_RR}
Suppose each dimension of the true parameter $\Params^\star$ is independently sampled from standard normal distribution $\Normal(0,1)$, $\forall t,a$ $\|\Context_{t,\Action}\|_2\leq 1$, $\eta_t$ is sampled from standard normal distribution $\Normal(0,1)$, the Bayesian reward regret of the \FairXLinTS\ algorithm is $\text{Bayes}\RewardRegret=\widetilde{O}\left( \Dimension\sqrt{\TimeSteps} \right)$. 
\end{theorem}

Similar to the \FairXTS\ algorithm in the MAB setting, the Bayesian fairness regret of  \FairXLinTS\ assumes a normal prior. Note that the Bayesian fairness regret of \FairXLinTS\ differs by order of $\sqrt{\Dimension}$ from the non-Bayesian fairness regret of the \FairXLinUCB\ algorithm. The Bayesian setting and the normal prior assumption enable us to explicitly bound the total variation distance between our policy and the optimal fair policy, which allows us to avoid going through the UCB-based analysis of the LinUCB algorithms as in the conventional way of proving Bayesian regret bound~\cite{russo2014learning}. 
\begin{figure*}[!htb]
\begin{subfigure}{.48\textwidth}
  \centering
  \includegraphics[width=\textwidth]{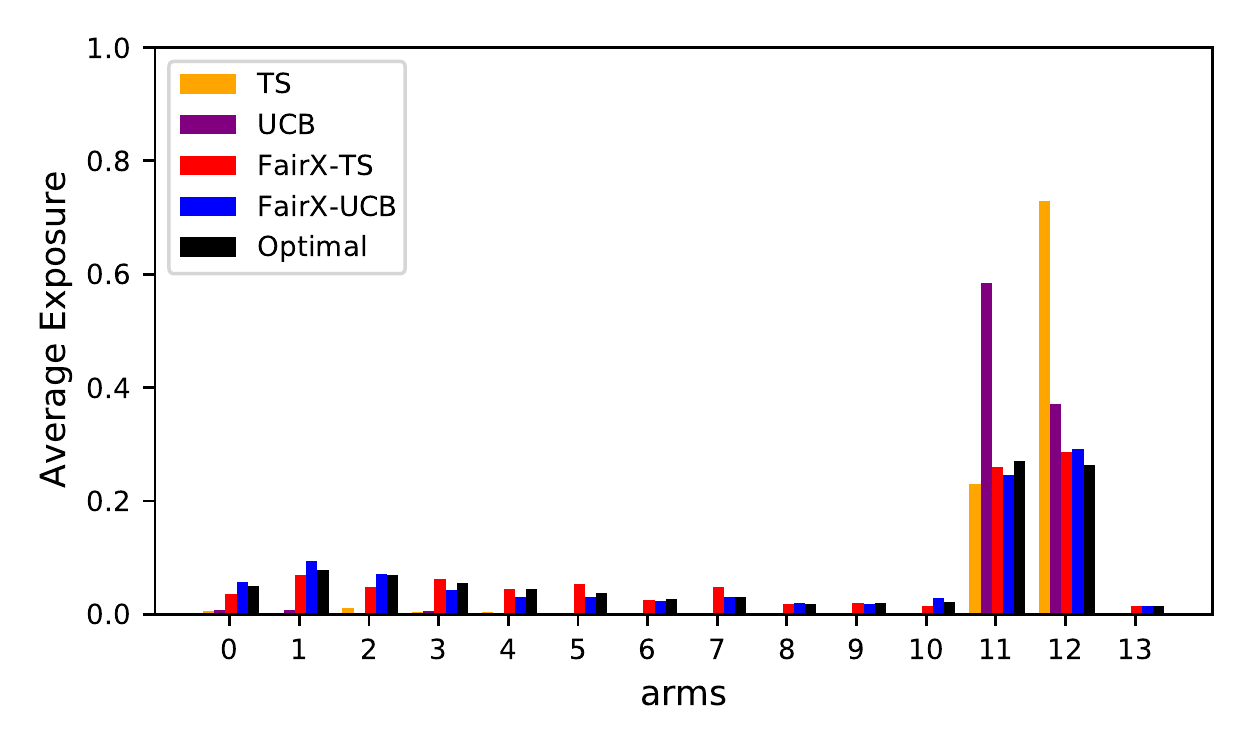}
\end{subfigure}
\begin{subfigure}{.48\textwidth}
  \centering
  \includegraphics[width=\linewidth]{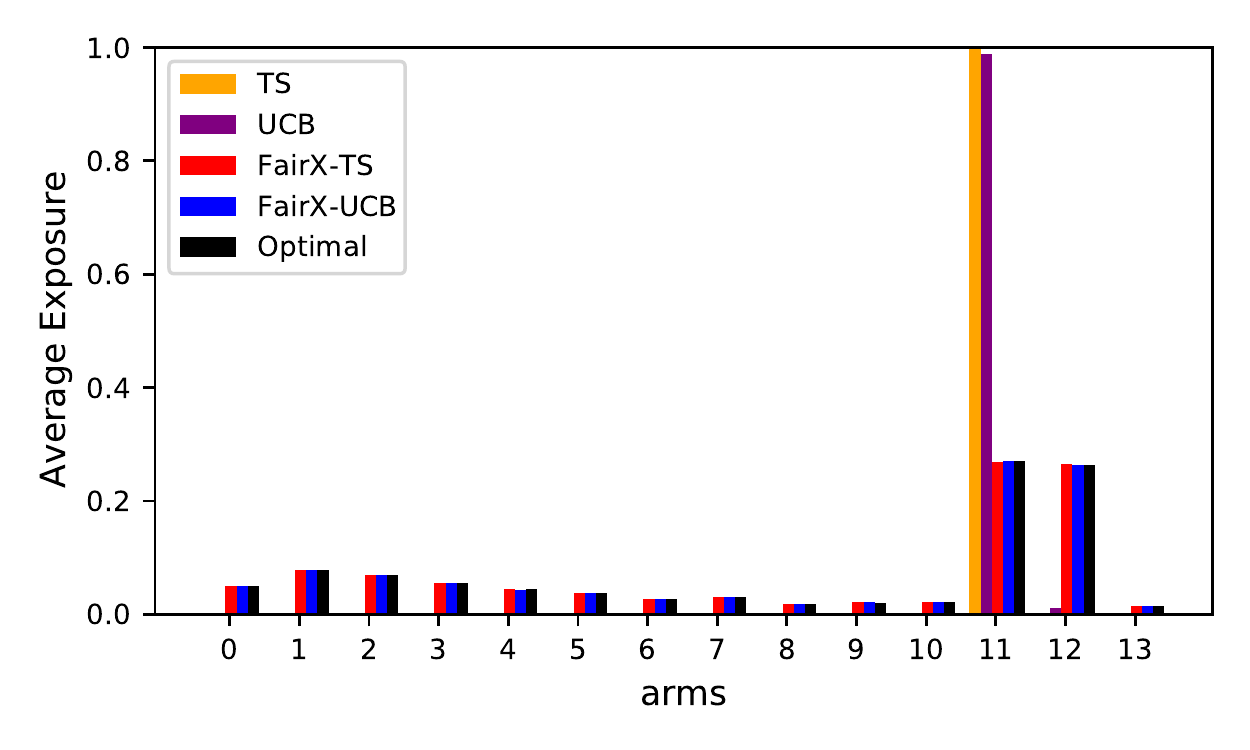}
\end{subfigure}
\caption{The average exposure distribution of different algorithms on the yeast dataset in the MAB setting after $2,000$ rounds (left) and $2,000,000$ rounds (right). ($c=4$)  }
\label{fig:hist}
\end{figure*}

\section{Experiments}

While the theoretical analysis provides worst case guarantees for the algorithms, we now evaluate  empirically how the algorithms perform on a range of tasks. We perform this evaluation both on synthetic and real-world data. The synthetic data allows us to control properties of the learning problem for internal validity, and the real-world data provides a data-point for external validity of the analysis.

\subsection{Experiment Setup}
\label{sec:exp_setup}

For the experiments where we control the properties of the synthetic data, we derive bandit problems from the multi-label datasets \emph{yeast}~\cite{horton1996probabilistic} and \emph{mediamill}~\cite{snoek2006challenge}. The yeast dataset consists of $2,417$ examples. Each example has $103$ features and belongs to one or multiple of the $14$ classes. We randomly split the dataset into two sets, 20\% as the validation set to tune hyper-parameters and 80\% as the test set to test the performance of different algorithms. For space reasons, the details and the results of the mediamill dataset are in Appendix~\ref{appendix:exp}. 
To simulate the bandit environment, we treat classes as arms and their labels ($0$ or $1$) as rewards. For each round $t$, the bandit environment randomly samples an example from the dataset. Then the bandit algorithm selects an arm (class), and its reward (class label) is revealed to the algorithm. To construct context vectors for the arms, we generate $50$-dimensional random Fourier features~\cite{rahimi2007random} from the outer product between the features of the example and the one-hot representation of the arms. 

For the experiments on real-world data, we use data from the Yahoo! Today Module~\cite{li2010contextual}, which contains user click logs from a news-article recommender system that was fielded for $10$ days. Each day logged around $4.6$ million events from a bandit that selected articles uniformly at random, which allows the use of the replay methodology~\cite{li2010contextual} for unbiased offline evaluation of new bandit algorithms. We use the data from the first day for hyper-parameter selection and report the results on the data from the second day. The results using all the data are presented in Appendix~\ref{appendix:exp}. Each article and each user is represented by a $6$-dimensional feature vector respectively. Following~\cite{li2010contextual}, we use the outer product between the user features and the article features as the context vector. 

To calculate the fairness and reward regret, we determine the optimal fair policy as follows. For MAB experiments, we use the empirical mean reward of each arm as the mean parameter for each arm. For linear bandit experiments, we fit a linear least square model that maps the context vector of each arm to its reward. Note that the linearity assumption does not necessarily hold for any of the datasets, and that rewards are known to change over time for the Yahoo! data. This realism adds a robustness component to the evaluation.

We also add straightforward \FairX-variants of the $\epsilon$-greedy algorithms to the empirical analysis, which we call \FairXEG\ and \FairXLinEG. The algorithms are identical to their conventional $\epsilon$-greedy counterparts, except that they construct their policies according to $\Policy_t(\Action)=\frac{\Merit(\hat{\Params}_{t,\Action})}{\sum_{\Action'}\Merit(\hat{\Params}_{t,\Action'})}$ or $\Policy_t(\Action)=\frac{\Merit(\hat{\Params}_{t}\cdot\Context_{t,\Action})}{\sum_{\Action'}\Merit(\hat{\Params}_{t}\cdot\Context_{t,\Action})}$  where $\hat{\Params}_{t}$ is the estimated parameter at round $t$. While $\epsilon$-greedy has weaker guarantees already in the conventional bandit setting, it is well known that it often performs well empirically and we thus add it as a reference for the more sophisticated algorithms.In addition to the \FairX\ algorithms, we also include the fairness regret of conventional UCB, TS, LinUCB, and LinTS  bandit algorithms. 

We use merit functions of the form $\Merit(\Params) = \exp(c\Params)$, since the choice of the constant $c$ provides a straightforward way to explore how the algorithms perform for steeper vs. flatter merit functions. In particular, the choice of $c$ varies the value of $\LipConst/\MeritMin$. For both \FairXUCB\ and \FairXLinUCB, we use projected gradient descent to solve the non-convex optimization problem each round. We set the learning rate to be $0.01$ and the number of steps to be $10$. For \FairXLinUCB, we use a fixed $\beta_t = \beta$ for all rounds. 

In general, we use grid search to tune hyper-parameters to minimize fairness regret on the validation set and report the performance on the test set. We grid search $\ConfiWidth$ for \FairXUCB\ and UCB; prior variance and reward variance for \FairXTS, TS, \FairXLinTS\, and LinTS; $\lambda$ and $\beta$ for \FairXLinUCB\ and LinUCB; $\epsilon$ for \FairXEG; $\epsilon$ and the regularization parameter of the ridge regression for \FairXLinEG. We run each experiment $10$ times and report the mean and the standard deviation.

\subsection{How unfair are conventional bandit algorithms?}

We first verify that conventional bandit algorithms indeed violate merit-based fairness of exposure, and that our \FairX\ algorithms specifically designed to ensure fairness do indeed perform better. \autoref{fig:hist} shows the average exposure that each arm received across rounds under the conventional UCB and TS algorithms for a typical run, and it compares this to the exposure allocation under the \FairXUCB\ and \FairXTS\ algorithm. The plots show the average exposure after 2,000 (left) and 2,000,000 (right) rounds, and it also includes the optimally fair exposure allocation. Already after 2,000 rounds, the conventional algorithms under-expose many of the arms. After 2,000,000 rounds, they focus virtually all exposure on arm 11, even though arm 12 has only slightly lower merit. Both \FairXUCB\ and \FairXTS\ track the optimal exposure allocation substantially better, and they converge to the optimally fair solution. This verifies that \FairX\ algorithms like \FairXUCB\ and \FairXTS\ are indeed necessary to enforce merit-based fairness of exposure. The following sections further show that conventional bandit algorithms consistently suffer from much larger fairness regret compared to \FairX\ algorithms across different datasets and merit functions in both MAB and linear bandits setting. 

\subsection{How do the \FairX\ algorithms compare in the MAB setting?}

\begin{figure}[!htb]
\begin{subfigure}{.235\textwidth}
  \centering
  \includegraphics[width=\linewidth]{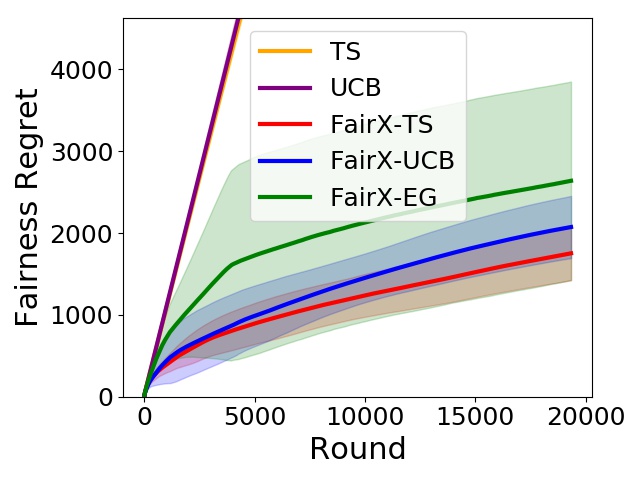}
\end{subfigure}
\begin{subfigure}{.235\textwidth}
  \centering
  \includegraphics[width=\linewidth]{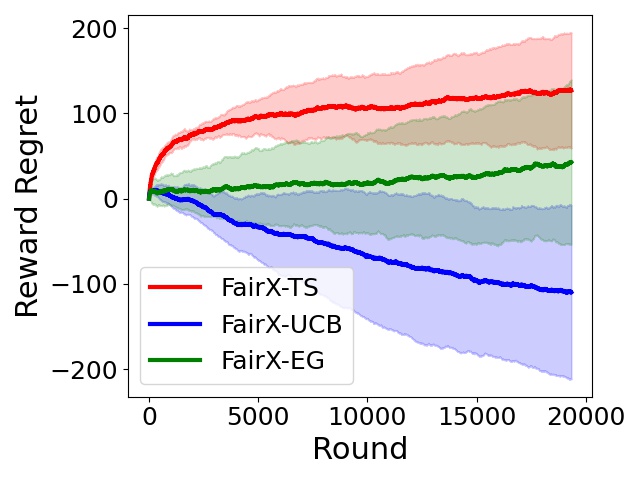}
\end{subfigure}
\caption{Fairness regret and reward regret of different MAB algorithms on the yeast dataset. ($c=10$)}
\label{fig:yeast_MAB}
\end{figure}

\autoref{fig:yeast_MAB} compares the performance of the bandit algorithms on the yeast dataset. The fairness regret converges roughly at the rate predicted by the bounds for \FairXUCB\ and \FairXTS, and \FairXEG\ shows a similar behavior as well. In terms of reward regret, all \FairX\ algorithms perform substantially better than their worst-case bounds suggest. Note that \FairXUCB\ does particularly well in terms of reward regret, but also note that part of this is due to violating fairness more than \FairXTS. Specifically, in the \FairX\ setting, an unfair policy can get better reward than the optimal fair policy, making a negative reward regret possible. While \FairXEG\ wins neither on fairness regret nor on reward regret, it nevertheless does surprisingly well given the simplicity of the exploration scheme. We conjecture that \FairXEG\ benefits from the implicit exploration that results from the stochasticity of the fair policies. Results for other merit functions are given in Appendix~\ref{appendix:exp}, and we find that the algorithms perform more similarly the flatter the merit function.

\subsection{How do the \FairX\ algorithms compare in the linear bandits setting?}

\begin{figure}[!htb]
\begin{subfigure}{.235\textwidth}
  \centering
  \includegraphics[width=\linewidth]{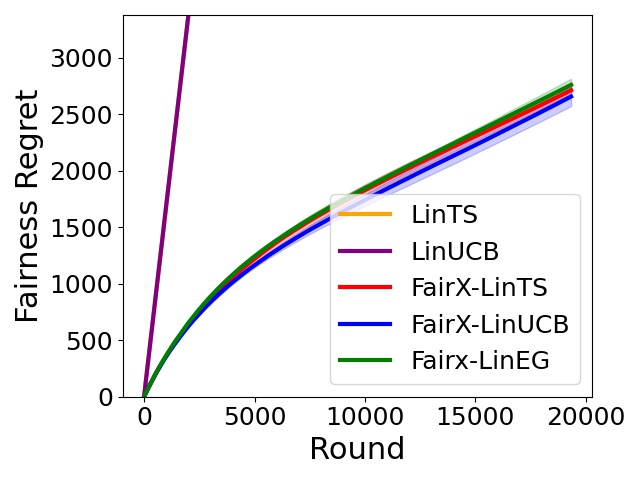}
\end{subfigure}
\begin{subfigure}{.235\textwidth}
  \centering
  \includegraphics[width=\linewidth]{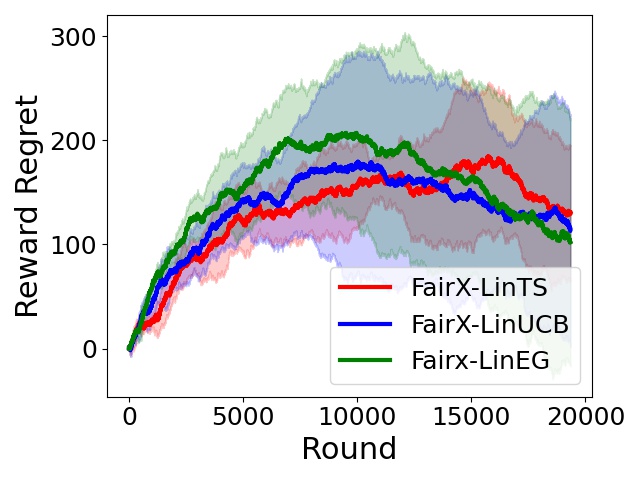}
\end{subfigure}
\caption{Fairness regret and reward regret of different linear bandit algorithms on the yeast dataset. ($c=3$) }
\label{fig:yeast_linear}
\end{figure}
\begin{figure*}[!tbh]
\begin{subfigure}{.24\textwidth}
  \centering
  \includegraphics[width=\linewidth]{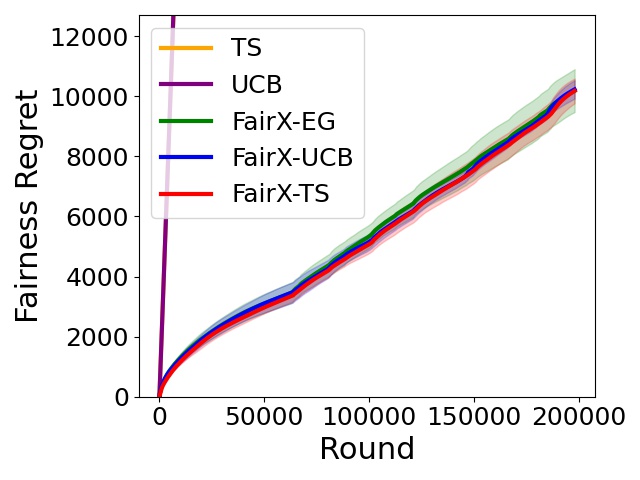}
\end{subfigure}
\begin{subfigure}{.24\textwidth}
  \centering
  \includegraphics[width=\linewidth]{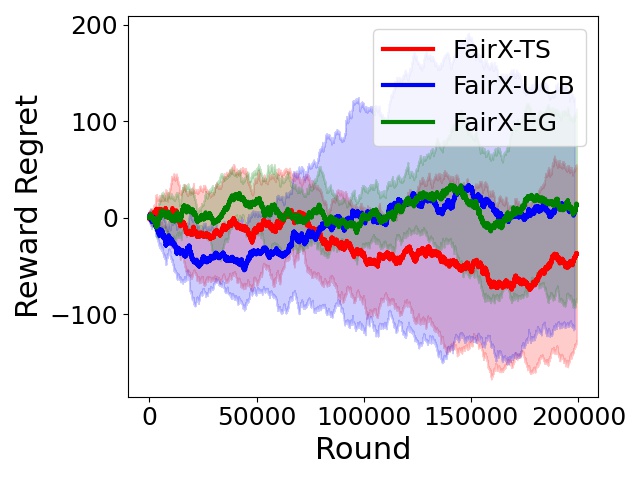}
\end{subfigure}
\begin{subfigure}{.24\textwidth}
  \centering
  \includegraphics[width=\linewidth]{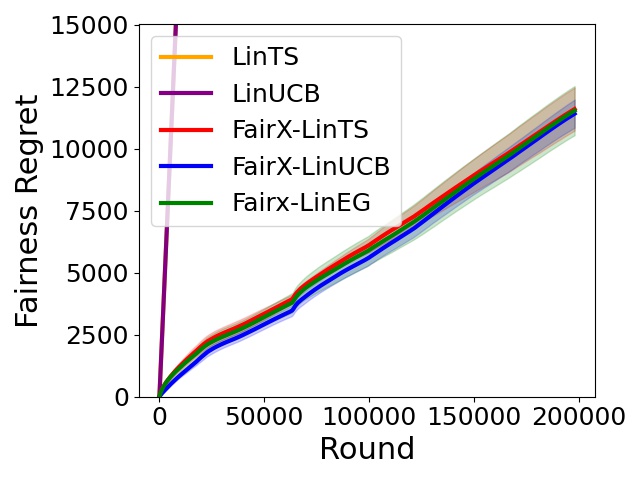}
\end{subfigure}
\begin{subfigure}{.24\textwidth}
  \centering
  \includegraphics[width=\textwidth]{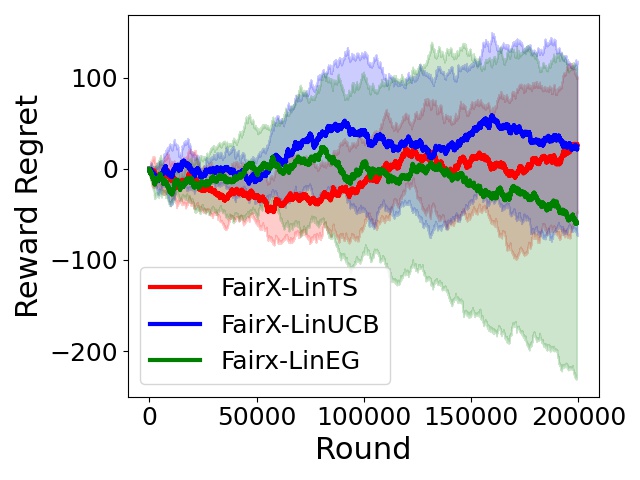}
\end{subfigure}
\caption{Experiment results on the Yahoo! dataset for both the MAB and the linear bandits setting ($c=10$ for both settings).}
\label{fig:yahoo}
\end{figure*}
We show the fairness regret and the reward regret of the bandit algorithms on the yeast dataset in \autoref{fig:yeast_linear}. Results for other merit functions are in Appendix~\ref{appendix:exp}. Similar to the MAB setting, there is no clear winner between the three \FairX\ algorithms. Again we see some trade-offs between reward regret and fairness regret, but all three \FairX\ algorithms show a qualitatively similar behavior. One difference is that the fairness regret no longer seems to converge. This can be explained with the misspecification of the linear model, as the estimated ``optimal'' policy that we use in our computation of regret may differ from the policy learned by the algorithms due to selection biases. Nevertheless, we conclude that the fairness achieved by the \FairX\ algorithms is still highly preferable to that of the conventional bandit algorithms.

\subsection{How do the \FairX\ algorithms compare on the real-world data?}

To validate the algorithms on a real-world application, \autoref{fig:yahoo} provides fairness and reward regret on the Yahoo! dataset for both the MAB and the linear bandits setting. Again, all three types of \FairX\ algorithms perform comparably and have reward regret that converges quickly. Note that even the MAB setting now includes some misspecification of the model, since the reward distribution changes over time. This explains the behavior of the fairness regret. However, all \FairX\ algorithms perform robustly in both settings, even though the real data does not exactly match the model assumptions.

\section{Conclusions}

We introduced a new bandit setting that formalizes merit-based fairness of exposure for both the stochastic MAB and the linear bandits setting. In particular, we define fairness regret and reward regret with respect to the optimal fair policy that fulfills the merit-based fairness of exposure, develop UCB and Thompson sampling algorithms for both settings, and prove bounds on their fairness and reward regret. An empirical evaluation shows that these algorithms provide substantially better fairness of exposure to the items, and that they are effective across a range of settings.

\section*{Acknowledgements}

This research was supported in part by NSF Awards IIS-1901168 and IIS-2008139. All content represents the opinion of the authors, which is not necessarily shared or endorsed by their respective employers and/or sponsors.




\bibliography{main}
\bibliographystyle{icml2021}






\newpage
\onecolumn
\appendix

\section{Proofs of the Theorems}
\label{section:proof_theorems}

\subsection{Proof of Theorem~\ref{theo:unique_optimal_fair_policy}}
\begin{proof}
First, the optimal fair policy satisfies the merit-based fairness of exposure constraints, since for any two arms $\Action,\Action'\in[\NumActions]$,
\[
\frac{\Policy^\star(\Action)}{\Merit(\Params^\star_\Action)} = \frac{\Policy^\star(\Action')}{\Merit(\Params^\star_{\Action'})} = \frac{1}{\sum_{\Action}\Merit(\Params^\star_\Action)}. 
\]

Second, we show that the solution is unique. The merit-based fairness of exposure constraints on $\Policy^\star$ correspond to $\NumActions-1$ linearly independent equations of $\Policy^\star$. With an additional linear equation $\sum_\Action\Policy^\star(\Action) =1$ that is linearly independent of the other $\NumActions-1$ ones, we have $\NumActions$ linearly independent equations on $\NumActions$ unknowns in $\Policy^\star$. Thus the solution to these equations is unique. 

\end{proof}

\subsection{Proof of Theorem~\ref{theo:lower_bound_min_merit}}
\begin{proof}
We prove the theorem by constructing two MAB instances and a $1$-Lipschitz merit function. We show the sum of the expected fairness regrets of the two MAB instances are linear in $\TimeSteps$ for any bandit algorithm under the merit function. Thus any bandit algorithm will have expected fairness regret linear in $\TimeSteps$ for at least one of the two MAB instances.

Let us consider two MAB instances $\BanditEnv^1 = (\RewardDist^1_1, \RewardDist^1_2)$ and $\BanditEnv^2 = (\RewardDist^2_1, \RewardDist^2_2)$ where each instance has two arms, and each arm has reward distributions being Gaussian distributions with variance fixed to be $1/2$. The first instance's mean $\mu_1 = (\theta, 2\theta)$, i.e. $\RewardDist^1_1=\Normal(\theta,1/2)$, $\RewardDist^1_2=\Normal(2\theta,1/2)$ and the second instance's mean is $\mu_2 = (2\theta, 2\theta)$, i.e. $\RewardDist^2_1=\Normal(2\theta,1/2)$, $\RewardDist^2_2=\Normal(2\theta,1/2)$,  where $\theta>0$ is a positive constant.  The merit function $\Merit$ is an identify function, i.e. $\Merit(\cdot) = \cdot$. This means that the optimal fair policy for the first instance is $\Policy^{\star, 1} = [ 1/3, 2/3 ]$, while the optimal fair policy for the second instance is $\pi^{\star,2} = [ 1/2, 1/2 ]$. Let us consider any bandit algorithm $\mathcal{A}$ which at every round $t$, produces a policy $\Policy_t$ (may be in a randomized way), based on the history $\History_t = \left(\Policy_1, \Action_1,\Reward_{1,\Action_1},\ldots, \Policy_{t-1}, \Action_{t-1},\Reward_{t-1,\Action_{t-1}} \right)$, i.e. $\pi_t \sim \mathcal{A}\left(\cdot  \vert \pi_1, a_1,\Reward_{1,\Action_1},\ldots, \Policy_{t-1}, \Action_{t-1}, \Reward_{t-1,\Action_{t-1}}\right), \Action_t\sim \Policy_t, \Reward_t \sim \RewardDist_{\Action_{t}}$.

Let us denote an outcome trajectory as $\tau = \left( \pi_1, a_1, r_{1,\Action_1},\ldots, \Policy_T, \Action_T, \Reward_{T,\Action_T} \right)$. Denote $\mathbb{P}^1$ as the distribution of $\tau$ of $\mathcal{A}$ interacting with the first MAB instance $\BanditEnv^1$, while $\mathbb{P}^2$ as the distribution of $\tau$ of $\mathcal{A}$ interacting with the second MAB instance $\BanditEnv^2$. The KL divergence between $\mathbb{P}^1$ and $\mathbb{P}^2$ can be upper bounded as follows: 
\begin{align*}
\text{KL}\left( \mathbb{P}^1, \mathbb{P}^2 \right) &  = \mathbb{E}_{\tau \sim \mathbb{P}^1}\left[ \ln \frac{ \mathbb{P}^1(\tau) }{ \mathbb{P}^2(\tau) } \right] =  \mathbb{E}_{\tau\sim \mathbb{P}^1} \left[ \sum_{t=1}^T  \ln \frac{ \RewardDist^1_{\Action_t}(\Reward_{\Action_t}) }{ \RewardDist^2_{\Action_t}(\Reward_{\Action_t})}  \right] = \sum_{t=1}^T \mathbb{E}_{\pi_t\sim \mathcal{A}^1} \mathbb{E}_{a_t \sim \pi_t} \text{KL}\left( \RewardDist^1_{\Action_t}, \RewardDist^2_{\Action_t} \right) \\
& = \sum_{t=1}^T\mathbb{E}_{\pi_t\sim \mathcal{A}^1} \left[  \pi_t(1)\text{KL}\left( \RewardDist^1_{1}, \RewardDist^2_{1} \right) \right] =\sum_{t=1}^T \mathbb{E}_{\pi_t\sim \mathcal{A}^1} \left[  \pi_t(1) \theta^2 \right]   \leq \sum_{t=1}^T \mathbb{E}_{\pi_t\sim \mathcal{A}^1}  \theta^2  = T \theta^2,
\end{align*} where $\pi_t\sim \mathcal{A}^1$ means that $\pi_t$ is sampled from the process of $\mathcal{A}$ interacting with the first MAB instance. 
Namely, when $\theta \to 0$, it would be hard to distinguish between $\mathbb{P}^1$ and $\mathbb{P}^2$.

For any sequence of policies $\pi_1,\dots, \pi_{T}$, we can lower bound the fairness regret for each instance as follows. For the fairness regret $\FairnessRegret^1_\TimeSteps$ of instance $\BanditEnv^1$, we have: 
\begin{align*}
\mE\left[\frac{1}{T}\FairnessRegret^1_\TimeSteps\right] & = \mE\left[\frac{1}{T}\sum_{t=1}^T \left( \left\lvert \pi_t(1) - 1/3  \right\rvert + \left\lvert \pi_t(2) - 2/3 \right\rvert \right)\right] \\
& \geq  \mE\left[\left\lvert \frac{1}{T}\sum_{t=1}^T \pi_t(1) - 1/3  \right\rvert + \left\lvert \frac{1}{T}\sum_{t=1}^T \pi_t(2) - 2/3  \right\rvert\right] =  2\mE\left[ \left\lvert \frac{1}{T}\sum_{t=1}^T \pi_t(1) - 1/3  \right\rvert\right].
\end{align*} Similarly, for the fairness regret $\FairnessRegret_\TimeSteps^2$ of instance $\BanditEnv^2$, we have:
\begin{align*}
\mE\left[\frac{1}{T} \FairnessRegret^2_\TimeSteps\right]  \geq 2\mE\left[ \left\lvert  \frac{1}{T} \sum_{t=1}^T \pi_1(t) - 1/2   \right\rvert\right].
\end{align*}

Thus we have:
\begin{align*}
\mE\left[\FairnessRegret^1_\TimeSteps / T\right] + \mE\left[\FairnessRegret^2_\TimeSteps / T\right] & \geq \frac{1}{6} \mathbb{P}^1\left( \frac{1}{T} \sum_{t=1}^T \pi_t(1) > \frac{5}{12} \right) + \frac{1}{6}  \mathbb{P}^2\left( \frac{1}{T} \sum_{t=1}^T \pi_t(1) \leq \frac{5}{12} \right) \\
&  \geq  \frac{1}{12} \exp\left( -\text{KL}\left( \mathbb{P}^1, \mathbb{P}^2 \right) \right) \geq \frac{1}{12} \exp\left( -\theta^2 T \right)
\end{align*}

where the second inequality applies the Bretagnolle-Huber inequality~\cite{bretagnolle1979estimation}. 

Set $\theta = 1/\sqrt{T}$, we prove that:
\begin{align*}
\mE\left[\FairnessRegret^1_\TimeSteps / T\right] + \mE\left[\FairnessRegret^2_\TimeSteps / T\right] \geq 0.03,
\end{align*} which implies that at least one instance suffers linear expected fairness regret. 

This concludes the proof.

\end{proof}

\subsection{Proof of Theorem~\ref{theo:lower_bound_lip_cont}}
\begin{proof}
Similar to the proof of Theorem ~\ref{theo:lower_bound_min_merit}, we construct two MAB instances and a merit function with minimum merit $1$. We show the sum of the expected fairness regrets of the two MAB instances are linear in $\TimeSteps$ for any bandit algorithm under this merit function. Thus any bandit algorithm will have expected fairness regret linear in $\TimeSteps$ for at least one of the two MAB instances. 

Let us consider two MAB instances $\BanditEnv^1 = (\RewardDist^1_1, \RewardDist^1_2)$ and $\BanditEnv^2 = (\RewardDist^2_1, \RewardDist^2_2)$ where each instance has two arms, and each arm has reward distributions being Gaussian distributions with variance fixed to be $1/2$. The first instance's mean $\mu_1 = (\theta, 0)$, i.e. $\RewardDist^1_1=\Normal(\theta,1/2)$, $\RewardDist^1_2=\Normal(
0,1/2)$ and the second instance's mean is $\mu_2 = (0, 0)$, i.e. $\RewardDist^2_1=\Normal(0,1/2)$, $\RewardDist^2_2=\Normal(0,1/2)$,  where $\theta>0$ is a positive constant to be set later.  The merit function $\Merit$ with minimum merit $1$ is a piece-wise linear function
\[
f(\Params) =\bigg\{\begin{array}{cc}
    1 & \Params\leq 0 \\
    \LipConst\Params + 1  & \Params > 0
\end{array}
\]

where $\LipConst>0$ is a positive constant to be set later. This means that the optimal fair policy for the first instance is $\Policy^{\star, 1} = [ (\LipConst\theta+1)/(\LipConst\theta+2), 1/(\LipConst\theta+2) ]$, while the optimal fair policy for the second instance is $\pi^{\star,2} = [ 1/2, 1/2 ]$. Let us consider any algorithm $\mathcal{A}$ which at every round $t$, produces a policy $\Policy_t$ (may be in a randomized way), based on the history $\History_t = \left(\Policy_1, \Action_1,\Reward_{1,\Action_1},\ldots, \Policy_{t-1}, \Action_{t-1},\Reward_{t-1,\Action_{t-1}} \right)$, i.e. $\pi_t \sim \mathcal{A}\left(\cdot  \vert \pi_1, a_1,\Reward_{1,\Action_1},\ldots, \Policy_{t-1}, \Action_{t-1}, \Reward_{t-1,\Action_{t-1}}\right), \Action_t\sim \Policy_t, \Reward_t \sim \RewardDist_{\Action_{t}}$.

Let us denote an outcome trajectory as $\tau = \left( \pi_1, a_1, r_{1,\Action_1},\ldots, \Policy_T, \Action_T, \Reward_{T,\Action_T} \right)$. Denote $\mathbb{P}^1$ as the distribution of $\tau$ of $\mathcal{A}$ interacting with the first MAB instance $\BanditEnv^1$, while $\mathbb{P}^2$ as the distribution of $\tau$ of $\mathcal{A}$ interacting with the second MAB instance $\BanditEnv^2$. The KL divergence between $\mathbb{P}^1$ and $\mathbb{P}^2$ can be upper bounded as follows: 
\begin{align*}
\text{KL}\left( \mathbb{P}^1, \mathbb{P}^2 \right) &  = \mathbb{E}_{\tau \sim \mathbb{P}^1}\left[ \ln \frac{ \mathbb{P}^1(\tau) }{ \mathbb{P}^2(\tau) } \right] =  \mathbb{E}_{\tau\sim \mathbb{P}^1} \left[ \sum_{t=1}^T  \ln \frac{ \RewardDist^1_{\Action_t}(\Reward_{\Action_t}) }{ \RewardDist^2_{\Action_t}(\Reward_{\Action_t})}  \right] = \sum_{t=1}^T \mathbb{E}_{\pi_t\sim \mathcal{A}^1} \mathbb{E}_{a_t \sim \pi_t} \text{KL}\left( \RewardDist^1_{\Action_t}, \RewardDist^2_{\Action_t} \right) \\
& = \sum_{t=1}^T\mathbb{E}_{\pi_t\sim \mathcal{A}^1} \left[  \pi_t(1)\text{KL}\left( \RewardDist^1_{1}, \RewardDist^2_{1} \right) \right] =\sum_{t=1}^T \mathbb{E}_{\pi_t\sim \mathcal{A}^1} \left[  \pi_t(1) \theta^2 \right]   \leq \sum_{t=1}^T \mathbb{E}_{\pi_t\sim \mathcal{A}^1}  \theta^2  = T \theta^2,
\end{align*} where $\pi_t\sim \mathcal{A}^1$ means that $\pi_t$ is sampled from the process of $\mathcal{A}$ interacting with the first MAB instance.

For any sequence of policies $\pi_1,\dots, \pi_{T}$, we can lower bound the fairness regret for each instance as follows. For the fairness regret $\FairnessRegret_\TimeSteps^1$ of instance $\BanditEnv^1$, we have: 
\begin{align*}
\mE\left[\frac{1}{T}\FairnessRegret^1_\TimeSteps\right] & = \frac{1}{T}\sum_{t=1}^T \left( \left\lvert \pi_t(1) -  (\LipConst\theta+1)/(\LipConst\theta+2) \right\rvert + \left\lvert \pi_t(2) -  1/(\LipConst\theta+2) \right\rvert \right) \\
& \geq  \left\lvert \frac{1}{T}\sum_{t=1}^T \pi_t(1) -  (\LipConst\theta+1)/(\LipConst\theta+2) \right\rvert + \left\lvert \frac{1}{T}\sum_{t=1}^T \pi_t(2) -  1/(\LipConst\theta+2)  \right\rvert\\
&=  2 \left\lvert \frac{1}{T}\sum_{t=1}^T \pi_t(1) -  (\LipConst\theta+1)/(\LipConst\theta+2) \right\rvert.
\end{align*} Similarly, for the fairness regret $\FairnessRegret_\TimeSteps^2$ of instance $\BanditEnv^2$, we have:
\begin{align*}
\mE\left[\frac{1}{T} \FairnessRegret^2_\TimeSteps\right]  \geq 2 \left\lvert  \frac{1}{T} \sum_{t=1}^T \pi_1(t) - 1/2   \right\rvert.
\end{align*}

Thus we have:
\begin{align*}
\mE\left[\FairnessRegret^1_\TimeSteps / T\right] + \mE\left[\FairnessRegret^2_\TimeSteps / T\right] & \geq \frac{\LipConst\theta}{2\LipConst\theta + 4} \mathbb{P}^1\left( \frac{1}{T} \sum_{t=1}^T \pi_t(1) \leq \frac{3\LipConst\theta+4}{4\LipConst\theta+8} \right) + \frac{\LipConst\theta}{2\LipConst\theta+4}  \mathbb{P}^2\left( \frac{1}{T} \sum_{t=1}^T \pi_t(1) > \frac{3\LipConst\theta+4}{4\LipConst\theta+8} \right) \\
&  \geq  \frac{\LipConst\theta}{4\LipConst\theta+8} \exp\left( -\text{KL}\left( \mathbb{P}^1, \mathbb{P}^2 \right) \right) \geq \frac{\LipConst\theta}{4\LipConst\theta+8} \exp\left( -\theta^2 T \right)
\end{align*}

where the second inequality applies the Bretagnolle-Huber inequality~\cite{bretagnolle1979estimation}. 

Set $\theta = 1/\sqrt{T}$ and $\LipConst=\sqrt{T}$, we prove that:
\begin{align*}
\mE\left[\FairnessRegret^1_\TimeSteps / T\right] + \mE\left[\FairnessRegret^2_\TimeSteps / T\right] \geq 0.03,
\end{align*} which implies that at least one instance suffers linear expected fairness regret. 

This concludes the proof.

\end{proof}

\subsection{Proof of Theorem~\ref{theo:fair_UCB_RR}}

\begin{lemma}
\label{lemma:fair_UCB_param_in_CR}
For any $\delta\in(0,1)$, with probability at least $1-\delta/2$, $\forall$ $t>\NumActions,\Action\in[\NumActions]$, $\Params^\star\in\ConfidenceRegion_t$. 
\end{lemma}
\begin{proof}
For any $t>\NumActions$ and $\Action\in[\NumActions]$, apply Hoeffding's inequality, we have with probability at least $1-\delta/(2\NumActions\TimeSteps)$, 
\[
\lvert\hat{\Params}_{t,\Action}-\Params^\star_\Action\rvert \leq \sqrt{ 2\ln(4\NumActions\TimeSteps/\delta)/\SelectedTimes_{t,\Action} }. 
\]
Apply union bound to $\forall$ $t>\NumActions,\Action\in[\NumActions]$, we conclude the proof. 
\end{proof}
\begin{lemma}
\label{lemma:fair_UCB_concentration_width}
For any $\delta\in(0,1)$, with probability at least $1-\delta/2$,
\[
\left\lvert\sum_{t=\NumActions+1}^\TimeSteps\mE_{\Action\sim\Policy_t}\sqrt{1/\SelectedTimes_{t,\Action}} - \sum_{t=\NumActions+1}^\TimeSteps\sqrt{1/\SelectedTimes_{t,\Action_t}}\right\rvert \leq \sqrt{2\TimeSteps\ln(4/\delta)}. 
\]
\end{lemma}
\begin{proof}
The sequence 
\[
\sqrt{1/\SelectedTimes_{t,\Action_t}} - \mE_{\Action\sim\Policy_t}\sqrt{1/\SelectedTimes_{t,\Action}}
\]

is a martingale difference sequence, and $\forall t>\NumActions$
\[
\left\lvert\sqrt{1/\SelectedTimes_{t,\Action_t}} - \mE_{\Action\sim\Policy_t}\sqrt{1/\SelectedTimes_{t,\Action}}\right\rvert\leq 1. 
\]
We can apply the Azuma-Hoeffding's inequality to get with probability at least $1-\delta/2$,

\[
\left\lvert\sum_{t=\NumActions+1}^\TimeSteps\mE_{\Action\sim\Policy_t}\sqrt{1/\SelectedTimes_{t,\Action}} - \sum_{t=\NumActions+1}^\TimeSteps\sqrt{1/\SelectedTimes_{t,\Action_t}}\right\rvert \leq \sqrt{2\TimeSteps\ln(4/\delta)}. 
\]

This concludes the proof. 

\end{proof}

\begin{proof}( Theorem~\ref{theo:fair_UCB_RR})

The reward regret can be upper bounded as follows:

\begin{align*}
\RewardRegret_\TimeSteps &= \sum_{t=1}^\TimeSteps\sum_\Action(\Policy^{\star}(\Action) -\Policy_t(\Action))\Params^{\star}_\Action \\
&\leq 2\NumActions + \sum_{t=\NumActions+1}^\TimeSteps\sum_\Action \Policy_t(\Action)\Params_{t,\Action} - \Policy_t(\Action)\Params_\Action^{\star}\\
&=2\NumActions + \sum_{t=\NumActions+1}^\TimeSteps\sum_\Action \Policy_t(\Action)(\Params_{t,\Action} - \hat{\Params}_{t,\Action} + \hat{\Params}_{t,\Action} -  \Params_\Action^{\star})\\
&\leq 2\NumActions + \sum_{t=\NumActions+1}^\TimeSteps\sum_\Action\Policy_t(\Action) 2\sqrt{ \frac{2\ln(4\TimeSteps\NumActions/\delta)}{\SelectedTimes_{t,\Action}} }\\
&= 2\NumActions + 2\sqrt{2\ln(4\TimeSteps\NumActions/\delta)}\sum_{t=\NumActions +1}^\TimeSteps\mE_{\Action\sim\Policy_t}\sqrt{1/\SelectedTimes_{t,\Action}}\\
&\leq 2\NumActions +  2\sqrt{2\ln(4\TimeSteps\NumActions/\delta)}\left( \sqrt{2\TimeSteps\ln(4/\delta)} + \sum_{t=\NumActions+1}^\TimeSteps\sqrt{1/\SelectedTimes_{t,\Action_t}}\right)\\
&\leq 2\NumActions +  2\sqrt{2\ln(4\TimeSteps\NumActions/\delta)}\left( \sqrt{2\TimeSteps\ln(4/\delta)} + 2\sqrt{\TimeSteps\NumActions}\right).
\end{align*}

The first inequality comes from Line~\ref{alg:fair_UCB:optimization} in Algorithm~\ref{alg:fair_UCB}. The second inequality comes from Lemma ~\ref{lemma:fair_UCB_param_in_CR}. The third inequality comes from Lemma~\ref{lemma:fair_UCB_concentration_width} and The last inequality applies the AM-GM inequality. Thus when $\TimeSteps>\NumActions$, with probability at least $1-\delta$, 
\[
\RewardRegret_\TimeSteps = \widetilde{O}\left(  \sqrt{\TimeSteps\NumActions}\right).
\]
This concludes the proof. 



\end{proof}

\subsection{Proof of Theorem~\ref{theo:fair_UCB_FR}}

\begin{proof}

For any $\delta\in(0,1)$, with probability at least $1-\delta$, the events in Lemma~\ref{lemma:fair_UCB_param_in_CR} and Lemma ~\ref{lemma:fair_UCB_concentration_width} hold and at each round $t>K$,
\begin{align*}
    &\sum_{\Action=1}^\NumActions\left\lvert \Policy_t(\Action) - \Policy^\star(\Action) \right\rvert\\
    =& \sum_{\Action=1}^\NumActions\left\lvert \frac{ \Merit(\Params_{t,\Action})}{ \sum_{\Action'=1}^{\NumActions} \Merit(\Params_{t,\Action}) } - \frac{ \Merit(\Params^{\star}_\Action) }{ \sum_{\Action'=1}^{\NumActions} \Merit(\Params^{\star}_{\Action'}) } \right\rvert\\
    =& \sum_{\Action=1}^\NumActions \frac{\left\lvert \Merit(\Params_{t,\Action})\sum_{\Action'=1}^{\NumActions} \Merit(\Params^{\star}_{\Action'}) - \Merit(\Params_\Action^{\star})\sum_{\Action'=1}^{\NumActions} \Merit(\Params_{t,\Action'})\right\rvert}{ \sum_{\Action'=1}^{\NumActions} \Merit(\Params_{t,\Action}) \sum_{\Action'=1}^{\NumActions} \Merit(\Params^{\star}_{\Action'})}\\
    =&\sum_{\Action=1}^\NumActions \frac{\left\lvert \Merit(\Params_{t,\Action})\sum_{\Action'=1}^{\NumActions} \Merit(\Params^{\star}_{\Action'}) - \Merit(\Params^\star_{\Action})\sum_{\Action'=1}^\NumActions\Merit(\Params^\star_{\Action'})  +  \Merit(\Params^\star_{\Action})\sum_{\Action'=1}^\NumActions\Merit(\Params^\star_{\Action'}) - \Merit(\Params_\Action^{\star})\sum_{\Action'=1}^{\NumActions} \Merit(\Params_{t,\Action'})\right\rvert}{ \sum_{\Action'=1}^{\NumActions} \Merit(\Params_{t,\Action}) \sum_{\Action'=1}^{\NumActions} \Merit(\Params^{\star}_{\Action'})}\\
    \leq&\frac{\sum_{\Action=1}^\NumActions\left\lvert \Merit(\Params_{t,\Action})-\Merit(\Params_{\Action}^{\star})\right\rvert \sum_{\Action'=1}^{\NumActions} \Merit(\Params^{\star}_{\Action'})  + \sum_{\Action=1}^\NumActions \Merit(\Params_{\Action}^{\star})\sum_{\Action'=1}^{\NumActions}\left\lvert \Merit(\Params^{\star}_{\Action'}) -\Merit(\Params_{t,\Action})\right\rvert }{ \sum_{\Action'=1}^{\NumActions} \Merit(\Params_{t,\Action}) \sum_{\Action'=1}^{\NumActions} \Merit(\Params^{\star}_{\Action'})}\\
    =&\frac{2\sum_{\Action=1}^\NumActions\left\lvert \Merit(\Params_{t,\Action})-\Merit(\Params_{\Action}^{\star})\right\rvert}{ \sum_{\Action'=1}^{\NumActions} \Merit(\Params_{t,\Action})}\\
    =&\frac{2\sum_{\Action=1}^\NumActions\frac{\Merit(\Params_{t,\Action})}{\Merit(\Params_{t,\Action})}\left\lvert \Merit(\Params_{t,\Action})-\Merit(\Params_{\Action}^{\star})\right\rvert}{ \sum_{\Action'=1}^{\NumActions} \Merit(\hat{\Params}_{t,\Action})}\\
    \leq&\sum_{\Action=1}^\NumActions\frac{4\LipConst\Policy_t(\Action)}{\MeritMin}\sqrt{2\ln(4\TimeSteps\NumActions/\delta)/\SelectedTimes_{t,\Action_t}} \\
    =&\frac{4\LipConst\sqrt{2\ln(4\TimeSteps\NumActions/\delta)}}{\MeritMin}\mE_{\Action\sim\Policy_t}\left[ 1/\sqrt{\SelectedTimes_{t,\Action_t}}\right]. 
\end{align*}

The second inequality comes from lemma~\ref{lemma:fair_UCB_param_in_CR}. And by lemma~\ref{lemma:fair_UCB_concentration_width},
\begin{align*}
\sum_{t=\NumActions+1}^\TimeSteps\sum_{\Action=1}^\NumActions\left\lvert \Policy_t(\Action) - \Policy^\star(\Action) \right\rvert
&\leq\frac{4\LipConst\sqrt{2\ln(4\TimeSteps\NumActions/\delta)}}{\MeritMin}\sum_{t=\NumActions+1}^\TimeSteps\mE_{\Action\sim\Policy_t}\left[ 1/\sqrt{\SelectedTimes_{t,\Action_t}}\right]\\
&\leq \frac{4\LipConst\sqrt{2\ln(4\TimeSteps\NumActions/\delta)}}{\MeritMin}\left( \sqrt{2\TimeSteps\ln(4/\delta)} + \sum_{t=\NumActions+1}^\TimeSteps\sqrt{1/\SelectedTimes_{t,\Action_t}} \right)\\
&\leq \frac{4\LipConst\sqrt{2\ln(4\TimeSteps\NumActions/\delta)}}{\MeritMin}\left( \sqrt{2\TimeSteps\ln(4/\delta)} + 2\sqrt{\TimeSteps\NumActions} \right). 
\end{align*}

So when $\TimeSteps>\NumActions$, with probability at least $1-\delta$, the fairness regret
\[
\FairnessRegret_\TimeSteps\leq 2\NumActions + \frac{4\LipConst\sqrt{2\ln(4\TimeSteps\NumActions/\delta)}}{\MeritMin}\left( \sqrt{2\TimeSteps\ln(4/\delta)} + 2\sqrt{\TimeSteps\NumActions} \right) = \widetilde{O}\left(  \LipConst\sqrt{\TimeSteps\NumActions}/\MeritMin \right),
\]
which concludes the proof. 


\end{proof}

\subsection{Proof of Theorem~\ref{theo:fair_LinUCB_FR}}
\begin{proposition} (Confidence)
\label{prop:confidence}
For $\delta\in(0,1)$, assume $\|\Params^\star\|_2\leq W$,  set $\beta_t =  \left(W + \sqrt{ \Dimension\ln(1+t/\Dimension)+2\ln(\pi^2t^2/3\delta) }\right)^2$ with probability at least $1-\delta/2$, $\forall t$, $\Params^\star\in \ConfidenceRegion_t$.
\end{proposition}
Section~\ref{sec:confidence_analysis} is devoted to establishing this confidence analysis. 

\begin{proposition}
\label{prop:sum_of_fr}
Let
\begin{equation}
\InstFairnessRegret_t = \sum_\Action\left\lvert\Policy^\star_t(\Action) - \Policy_t(\Action)\right\rvert
\end{equation}
denote the instantaneous fairness regret acquired by the algorithm at round t. For the \FairXLinUCB\ algorithm, if $\Params^\star\in \ConfidenceRegion_t$ for all $t\leq \TimeSteps$, then with probability at least $1-\delta/2$
\begin{equation}
    \sum_{t=1}^{\TimeSteps} \InstFairnessRegret_t \leq \frac{4\LipConst\sqrt{\beta_\TimeSteps}}{\MeritMin}\sqrt{2\TimeSteps\Dimension\ln(1+\frac{\TimeSteps}{\Dimension}) } + \frac{4\LipConst\sqrt{\beta_\TimeSteps}}{\MeritMin}\sqrt{2\TimeSteps\ln(4/\delta)}.
\end{equation}
\end{proposition}

\begin{lemma}[Lemma 7 in~\cite{dani2008stochastic} ]
\label{lemma:width}
 For the \FairXLinUCB\ algorithm, if $\Params\in \ConfidenceRegion_t$, then for any $\Context\in\mathbbm{R}^\Dimension$ 
\begin{displaymath}
\lvert(\Params-\hat{\Params}_t)\cdot\Context\rvert\leq\sqrt{\beta_t\Context^\top\CovMatrix_t^{-1}\Context}. 
\end{displaymath}
\end{lemma}

Define 
\begin{displaymath}
\Width_{t,\Action}\coloneqq\sqrt{\Context_{t,\Action}^\top\CovMatrix_t^{-1}\Context_{t,\Action}}
\end{displaymath}
which we interpret as the ``normalized width'' at time t for action $\Action$. 

\begin{lemma}
\label{lemma:fair_LinUCB_concentration_width}
For the \FairXLinUCB\ algorithm, with probability $1-\delta/2$, 
\[
\left\lvert\sum_{t=1}^{\TimeSteps}\Width_{t,\Action_t} - \sum_{t=1}^\TimeSteps\mE_{\Action\sim\Policy_t}\Width_{t,\Action}\right\rvert\leq \sqrt{ 2\TimeSteps\ln(4/\delta) }. 
\]
\end{lemma}

\begin{proof}
The sequence 
\[
\Width_{t,\Action_{t}} - \mE_{\Action\sim\Policy_t}\Width_{t,\Action}
\]
is a martingale difference sequence 
and $\forall t$
\[
\Width_{t,\Action} = \|\Context_{t,\Action_t}\|_{\CovMatrix_t^{-1}}\leq \sqrt{\lambda_{\max}(\CovMatrix_t^{-1})}\|\Context_{t,\Action_t}\|_2\leq 1, 
\]
where $\lambda_{max}(\cdot)$ denotes the largest eigenvalue of a matrix.  Using Azuma-Hoeffding's inequality, with probability at least $1-\delta/2$
\[
\left\lvert\sum_{t=1}^{\TimeSteps}\Width_{t,\Action_t} - \sum_{t=1}^\TimeSteps\mE_{\Action\sim\Policy_t}\Width_{t,\Action}\right\rvert\leq \sqrt{ 2\TimeSteps\ln(4/\delta) }.
\]

\end{proof}

\begin{lemma}
For the \FairXLinUCB\ algorithm , if $\forall t$, $\Params^\star\in \ConfidenceRegion_t$, then with probability at least $1-\delta/2$
\[
\sum_{t=1}^{\TimeSteps} \InstFairnessRegret_t \leq \frac{4\LipConst\sqrt{\beta_t}}{\MeritMin}\sum_{t=1}^{\TimeSteps}\Width_{t,\Action_t}+ \frac{4\LipConst\sqrt{\beta_t}}{\MeritMin}\sqrt{2\TimeSteps\ln(4/\delta)}
\]
\end{lemma}

\begin{proof}
\begin{displaymath}
\begin{split}
&\InstFairnessRegret_t\\
=& \sum_\Action\left\lvert \frac{\Merit(\Params^\star\cdot\Context_{t,\Action})}{\sum_{\Action'}\Merit(\Params^\star\cdot\Context_{t,\Action'})} - \frac{\Merit(\Params_t\cdot\Context_{t,\Action})}{\sum_{\Action'}\Merit(\Params_t\cdot\Context_{t,\Action'})} \right\rvert\\
=& \sum_\Action\left\lvert \frac{\Merit(\Params^\star\cdot\Context_{t,\Action})\sum_{\Action'}\Merit(\Params_t\cdot\Context_{t,\Action'})-\Merit(\Params_t\cdot\Context_{t,\Action})\sum_{\Action'}\Merit(\Params^\star\cdot\Context_{t,\Action'})}{\sum_{\Action'}\Merit(\Params^\star\cdot\Context_{t,\Action'})\sum_{\Action'}\Merit(\Params_t\cdot\Context_{t,\Action'})}\right\rvert\\
=& \sum_\Action\left\lvert \frac{\Merit(\Params^\star\cdot\Context_{t,\Action})\sum_{\Action'}\Merit(\Params_t\cdot\Context_{t,\Action'}) - \Merit(\Params^\star\cdot\Context_{t,\Action})\sum_{\Action'}\Merit(\Params^\star\cdot\Context_{t,\Action'}) + \Merit(\Params^\star\cdot\Context_{t,\Action})\sum_{\Action'}\Merit(\Params^\star\cdot\Context_{t,\Action'}) -\Merit(\Params_t\cdot\Context_{t,\Action})\sum_{\Action'}\Merit(\Params^\star\cdot\Context_{t,\Action'})}{\sum_{\Action'}\Merit(\Params^\star\cdot\Context_{t,\Action'})\sum_{\Action'}\Merit(\Params_t\cdot\Context_{t,\Action'})}\right\rvert\\
\leq& \sum_\Action \frac{\left\lvert \Merit(\Params^\star\cdot\Context_{t,\Action})\sum_{\Action'}\Merit(\Params_t\cdot\Context_{t,\Action'}) - \Merit(\Params^\star\cdot\Context_{t,\Action})\sum_{\Action'}\Merit(\Params^\star\cdot\Context_{t,\Action'})\right\rvert + \left\lvert \Merit(\Params^\star\cdot\Context_{t,\Action})\sum_{\Action'}\Merit(\Params^\star\cdot\Context_{t,\Action'}) -\Merit(\Params_t\cdot\Context_{t,\Action})\sum_{\Action'}\Merit(\Params^\star\cdot\Context_{t,\Action'})\right\rvert}{\sum_{\Action'}\Merit(\Params^\star\cdot\Context_{t,\Action'})\sum_{\Action'}\Merit(\Params_t\cdot\Context_{t,\Action'})}\\
\leq&  \frac{2\sum_\Action\left\lvert \Merit(\Params^\star\cdot\Context_{t,\Action}) - \Merit(\Params_t\cdot\Context_{t,\Action})\right\rvert }{\sum_{\Action'}\Merit(\Params_t\cdot\Context_{t,\Action'})}\\
=&  2\sum_\Action\frac{\Policy_t(\Action)}{\Merit(\Params_t\cdot\Context_{t,\Action})} \left\lvert \Merit(\Params^\star\cdot\Context_{t,\Action}) - \Merit(\hat{\Params}_t\cdot\Context_{t,\Action}) + \Merit(\hat{\Params}_t\cdot\Context_{t,\Action}) - \Merit(\Params_t\cdot\Context_{t,\Action})\right\rvert\\
\leq& \frac{2\LipConst}{\MeritMin} \mE_{\Action\sim\Policy_t}\left[\|\Params^\star -\hat{\Params}_t\|_{\CovMatrix_t}\|\Context_{t,\Action}\|_{\CovMatrix_t^{-1}} + \|\Params_t -\hat{\Params}_t\|_{\CovMatrix_t}\|\Context_{t,\Action}\|_{\CovMatrix_t^{-1}}\right]\\
\leq& \frac{4\LipConst\sqrt{\beta_t}}{\MeritMin}\mE_{\Action\sim\Policy_t}\Width_{t,\Action}. 
\end{split}
\end{displaymath}

So by lemma~\ref{lemma:fair_LinUCB_concentration_width} and that $\beta_t$ is increasing,
\[
\sum_{t=1}^{\TimeSteps} \InstFairnessRegret_t \leq \frac{4\LipConst\sqrt{\beta_\TimeSteps}}{\MeritMin}\sum_{t=1}^{\TimeSteps}\Width_{t,\Action_t} + \frac{4\LipConst\sqrt{\beta_\TimeSteps}}{\MeritMin}\sqrt{2\TimeSteps\ln(4/\delta)}. 
\]
\end{proof}

\begin{lemma} [Lemma 10 in~\cite{dani2008stochastic}]
We have $\forall\;t$
\begin{displaymath}
\det \CovMatrix_{t} = \prod_{\tau=1}^{t-1}( 1+\Width_{t,\Action_t}^2).
\end{displaymath}
\end{lemma}
\begin{lemma}
\label{lemma:det}
$\forall t$, $\det \CovMatrix_{t+1}\leq (1+t/\Dimension)^\Dimension$.  
\end{lemma}
\begin{proof}
\begin{equation}
\begin{split}
    \text{Trace } \CovMatrix_{t+1} &= \text{Trace }\left( I + \sum_{\tau=1}^{t}\Context_{\tau,\Action_{\tau}}\Context_{\tau,\Action_{\tau}}^\top\right)\\
    &=\Dimension + \sum_{\tau=1}^{t=1}\text{Trace }\left(\Context_{t,\Action_t} \Context_{t,\Action_t}^\top\right)\\
    &=\Dimension + \sum_{\tau=1}^{t} \|\Context_{t,\Action_t}\|_2^2\\
    &\leq \Dimension+t. 
\end{split}
\end{equation}
Now, recall that Trace $\CovMatrix_t$ equals the sum of the eigenvalues of $\CovMatrix_t$. On the other hand, $\det(\CovMatrix_t)$ equals the product of the eigenvalues. Since $\CovMatrix_t$ is positive definite, its eigenvalues are all positive. Subject to these constraints, by AM-GM inequality, $\det(\CovMatrix_t)$ is maximized when all the eigenvalues are equal; the desired bound follows. 
\end{proof}

\begin{lemma}
\label{lemma:sum_of_width_squares}
We have for all $t$,
\begin{displaymath}
\sum_{\tau=1}^t w_{t,\Action_t}^2\leq  2\Dimension\ln\left( 1 + t/\Dimension \right). 
\end{displaymath}
\end{lemma}
\begin{proof}
Using the fact that for $ 0\leq y\leq 1$, $ln(1+y)\geq y/2$, we have
\begin{displaymath}
\begin{split}
\sum_{\tau =1}^t \Width_{t,\Action}^2&\leq 2\sum_{\tau=1}^t\ln(1+\Width^2_{t,\Action_t})\\
&=2\ln(\det(\CovMatrix_{t+1}))\\
&\leq 2\Dimension\ln( 1 + t/\Dimension )
\end{split}
\end{displaymath}
by the previous two lemmas.
\end{proof}

\begin{proof} ( Proof of Proposition~\ref{prop:sum_of_fr})
\begin{align*}
\sum_{t=1}^{\TimeSteps} \InstFairnessRegret_t &\leq \frac{4\LipConst\sqrt{\beta_\TimeSteps}}{\MeritMin}\sum_{t=1}^{\TimeSteps}\Width_{t,\Action_t} + \frac{4\LipConst\sqrt{\beta_\TimeSteps}}{\MeritMin}\sqrt{2\TimeSteps\ln(4/\delta)}\\
&\leq \frac{4\LipConst\sqrt{\beta_\TimeSteps}}{\MeritMin}\sqrt{T\sum_{t=1}^{\TimeSteps}\Width^2_{t,\Action_t}} + \frac{4\LipConst\sqrt{\beta_\TimeSteps}}{\MeritMin}\sqrt{2\TimeSteps\ln(4/\delta)}\\
&\leq \frac{4\LipConst\sqrt{\beta_\TimeSteps}}{\MeritMin}\sqrt{2\TimeSteps\Dimension\ln(1+\frac{\TimeSteps}{\Dimension}) } + \frac{4\LipConst\sqrt{\beta_\TimeSteps}}{\MeritMin}\sqrt{2\TimeSteps\ln(4/\delta)}. 
\end{align*}

\end{proof}

\begin{proof}(Proof of theorem~\ref{theo:fair_LinUCB_FR})

By Proposition~\ref{prop:confidence} and Proposition~\ref{prop:sum_of_fr}, with probability at least $1-\delta$, 
\begin{align*}
\FairnessRegret_\TimeSteps = \sum_{t=1}^\TimeSteps\InstFairnessRegret_t
\leq \frac{4\LipConst\sqrt{\beta_\TimeSteps}}{\MeritMin}\sqrt{2\TimeSteps\Dimension\ln(1+\frac{t}{\Dimension}) } + \frac{4\LipConst\sqrt{\beta_\TimeSteps}}{\MeritMin}\sqrt{2\TimeSteps\ln(4/\delta)} = \widetilde{O}\left(\LipConst\Dimension\sqrt{\TimeSteps}/\MeritMin\right). 
\end{align*}


\end{proof}

\subsubsection{Confidence analysis}
\label{sec:confidence_analysis}
In this section, we prove Proposition~\ref{prop:confidence}, which states that with high probability, the true parameter $\Params^{\star}$ lies in the confidence Region $\ConfidenceRegion_t$  for all $t$. 

\begin{proof}(Proof of Proposition~\ref{prop:confidence})

Since $\Reward_{\tau,\Action_\tau}=\Params^{\star}\cdot\Context_{\tau,\Action_\tau} + \eta_\tau$, we have
\begin{align*}
\hat{\Params}_t-\Params^{\star} = \CovMatrix_{t}^{-1}\sum_{\tau=1}^{t-1}\Reward_{\tau,\Action_\tau}\Context_{\tau,\Action_\tau} - \Params^{\star} =  \CovMatrix_{t}^{-1}\Params^{\star} + \CovMatrix_{t}^{-1}\sum_{\tau=1}^{t-1}\eta_\tau \Context_{\tau,\Action_{\tau}}
\end{align*}

For any $0<\delta_t<1$, using self-normalized bound for vector-valued martingales~\cite{abbasi2011improved} and Lemma~\ref{lemma:det}, we have with probability at least $1-\delta_t$,
\begin{align*}
\sqrt{(\hat{\Params}_t-\Params^{\star})^{\top}\CovMatrix_{t}(\hat{\Params}_t-\Params^{\star})}&=\| \CovMatrix_t^{1/2}(\hat{\Params}_t-\Params^{\star}) \|_2\\
&\leq \|\CovMatrix_{t}^{-1/2}\Params^{\star}\|_2 + \|\CovMatrix_{t}^{-1/2}\sum_{\tau=1}^{t-1}\eta_\tau \Context_{\tau,\Action_{\tau}}\|_2\\
&\leq \|\Params^{\star}\|_2 + \sqrt{\ln(\det(\CovMatrix_{t})\det(\CovMatrix_1)^{-1}/\delta_t^2)}\\
&\leq W + \sqrt{ \Dimension\ln(1+t/\Dimension)+2\ln(1/\delta_t) }. 
\end{align*}

Let $\delta_t = \frac{3\delta/\pi^2}{t^2}$, 
\[
\mathbb{P}(\forall t\;\Params^{\star}\in \ConfidenceRegion_t) \geq 1- \sum_{t=1}^{\infty} (\delta/t^2)(3/\pi^2)= 1- \frac{\delta}{2}.  
\]

\end{proof}

\subsection{Proof of Theorem~\ref{theo:fair_LinUCB_RR}}

\begin{proof}
With probability at least $1-\delta$, the events in Proposition~\ref{prop:confidence} and Lemma~\ref{lemma:fair_LinUCB_concentration_width} hold and
\begin{align*}
&\RewardRegret_\TimeSteps\\
=&\sum_{t=1}^{\TimeSteps}\Params^{\star}\cdot\mE_{\Action\sim\Policy_t^{\star}}[\Context_{t,\Action}] - \Params^{\star}\cdot\mE_{\Action\sim\Policy_t}[\Context_{t,\Action}]\\
\leq&\sum_{t=1}^{\TimeSteps} \Params_t\cdot\mE_{\Action\sim\Policy_t}[\Context_{t,\Action}]-\Params^{\star}\cdot\mE_{\Action\sim\Policy_t}[\Context_{t,\Action}]\\
=&\sum_{t=1}^{\TimeSteps} (\Params_t-\hat{\Params}_t)\cdot\mE_{\Action\sim\Policy_t}[\Context_{t,\Action}] + (\hat{\Params}_t-\Params^{\star})\cdot\mE_{\Action\sim\Policy_t}[\Context_{t,\Action}]\\
\leq& \sum_{t=1}^{\TimeSteps}\mE_{\Action\sim\Policy_t}[2\sqrt{\beta_t}\Width_{t,\Action}]\\
\leq& 2\sqrt{\beta_\TimeSteps}\sum_{t=1}^{\TimeSteps}\mE_{\Action\sim\Policy_t}[\Width_{t,\Action}]\\
\leq& 2\sqrt{\beta_\TimeSteps}\left(\sum_{t=1}^{\TimeSteps}\Width_{t,\Action_t} + 2\sqrt{2\TimeSteps\ln(4/\delta)}\right)\\
\leq& 2\sqrt{\beta_\TimeSteps}\left(\sqrt{2\TimeSteps\Dimension\ln(1+\frac{\TimeSteps}{\Dimension})}+ 2\sqrt{2\TimeSteps\ln(4/\delta)}\right)\\
=&\widetilde{O}\left(\Dimension\sqrt{\TimeSteps}\right). 
\end{align*}

The first inequality comes from the algorithm. The second inequality comes from Lemma~\ref{lemma:width}. The third inequality comes from the fact that $\beta_t$ is increasing. The fourth inequality comes from lemma~\ref{lemma:fair_LinUCB_concentration_width}. And the last inequality comes from Lemma ~\ref{lemma:sum_of_width_squares}. 



\end{proof}

\subsection{Proof of Theorem~\ref{theo:fair_LinTS_FR}}

\begin{proof}
 Denote the posterior distribution of $\Params^\star$ conditioned on $\History_t$ as $p(\Params^\star \vert \History_t)$ and the corresponding conditional expectation as $\mE[ \cdot \vert \History_t ]$. In stochastic linear bandits, the posterior distribution is a Gaussian distribution: $p(\cdot | \History_t) \coloneqq \Normal\left( \hat{\Params}_t, \CovMatrix^{-1}_t \right)$. 
We notice that our $\Params_t$ and $\Params^\star$ are identically distributed from $p(\cdot | \History_t)$. 

First, We can follow the same step we had above in the proof of fairness regret of \FairXLinUCB\ to upper bound the instantaneous fairness regret as follows (conditioned on history $\History_t$):
\begin{align*}
\mE[\InstFairnessRegret_t] \leq  \frac{2}{\MeritMin}\mE_{\History_t}\left[\mE_{\Params_t,\Params^\star}[\mE_{\Action\sim \Policy_t} \left\lvert \Merit(\Params_t \cdot \Context_{t,\Action}) - \Merit(\Params^\star \cdot \Context_{t,\Action}) \right\rvert\mid\History_t]\right] \leq \frac{2\LipConst}{\MeritMin} \mE_{\History_t}\left[\mE_{\Params_t,\Params^\star}[\mE_{\Action\sim \Policy_t} \left\lvert \Params_t \cdot \Context_{t,\Action} - \Params^\star \cdot \Context_{t,\Action} \right\rvert\mid\History_t]\right]
\end{align*}

Note that $\Policy_t$ is fully determined by $\Params_t$ and is independent of $\Params^\star$ given $\History_t$. In the following, We use $\Policy_{\Params_t}$ to denote $\Policy_t$ to stress the dependence of $\Policy_t$ on $\Params_t$. Hence, taking expectation with respect to the randomness of $\Params_t$ and $\Params^\star$ :
\begin{align*}
\mE_{\Params_t, \Params^\star}\left[ \sum_\Action \Policy_{
\Params_t}(\Action) \left\lvert (\Params_t - \Params^\star)\cdot \Context_{t,\Action} \right\rvert  \bigg\vert\History_t\right] =  \mE_{\Params_t} \left[ \sum_{\Action} \Policy_{\Params_t}(\Action) \mE_{\Params^\star} \left[\left\lvert (\Params_t - \Params^\star)\cdot \Context_{t,\Action}\right\rvert\mid\History_t,\Params_t\right]\bigg\vert\History_t \right]
\end{align*}
Note that for any $\Context_{t,\Action}$, conditioned on $\Params_t$, we have:
\begin{align*}
(\Params_t - \Params^\star)\cdot \Context_{t,\Action} \sim \Normal\left( \Params_{t}\cdot \Context_{t,\Action} - \hat{\Params}_t\cdot \Context_{t,\Action},  \Context_{t,\Action}^{\top} \CovMatrix_{t}^{-1} \Context_{t,\Action}    \right),
\end{align*}

which means that:
\begin{align*}
\mE_{\Params^\star} [\left\lvert (\Params_t - \Params^\star)\cdot \Context_{t,\Action} \right\rvert\mid\History_t,\Params_t]& \leq \sqrt{ \mE_{\Params^\star}  \left[\left\lvert (\Params_t - \Params^\star)\cdot \Context_{t,\Action} \right\rvert^2\mid\History_t,\Params_t\right] } = \sqrt{  \Context_{t,\Action}^\top \CovMatrix_{t}^{-1} \Context_{t,\Action} + \left(( \Params_t - \hat{\Params}_t )\cdot \Context_{t,\Action}\right)^2     } \\
& \leq \sqrt{ \Context_{t,\Action}^{\top} \CovMatrix_t^{-1} \Context_{t,\Action} } + \left\lvert (\Params_t- \hat{\Params}_t)\cdot \Context_{t,\Action}  \right\rvert 
\end{align*} 
Denote the random variable $z_{t,\Action} \coloneqq (\Params_t - \hat{\Params}_t) \cdot \Context_{t,\Action} / \sqrt{ \Context_{t,\Action}^{\top} \CovMatrix_{t}^{-1} \Context_{t,\Action}} $. Given $\History_t$, $z_{t,\Action}$ is a random variable and is only dependent on $\Params_t$.  Note that $z_{t,\Action} \sim \Normal\left( 0, 1  \right)$. Thus by the CDF of normal distribution, we have 
with probability at least $1-\delta'$:
\begin{align*}
\lvert (\Params_t - \hat{\Params}_t)\cdot \Context_{t,\Action} \rvert \leq \sqrt{2 \ln(1/\delta')}  \sqrt{ \Context_{t,\Action}^{\top} \CovMatrix_t^{-1} \Context_{t,\Action}}.
\end{align*} Allow union bound over all $a$ and all $T$, we get with probability at least $1-\delta'$:
\begin{align*}
\forall t\in[\TimeSteps], \Context_{t,\Action}\in \Actions_t:  \lvert (\Params_t - \hat{\Params}_t) \cdot \Context_{t,\Action} \rvert \leq \sqrt{2 \ln(\NumActions\TimeSteps / \delta' )} \sqrt{ \Context_{t,\Action}^{\top} \CovMatrix_t^{-1}  \Context_{t,\Action}}.
\end{align*} Denote the above inequality at episode $t$ as event $\mathcal{E}_t$ (note that $\mathcal{E}_t$  only depends on the random variable $\Params_t$). We have
\begin{align*}
& \mE_{\Params_t} \left[ \sum_\Action \Policy_{\Params_t}(\Action) \mE_{\Params^\star}[ \lvert(\Params_t- \Params^\star) \cdot \Context_{t,\Action} \rvert\mid\History_t,\Params_t]\bigg\vert\History_t \right]  \\
=& \mE_{\Params_t} \left[  \Indicator\{\mathcal{E}_t \}  \sum_\Action \Policy_{\Params_t}(\Action) \mE_{\Params^\star}[ \lvert(\Params_t- \Params^\star) \cdot \Context_{t,\Action} \rvert\mid\History_t,\Params_t]\bigg\vert\History_t  \right]  + \mE_{\Params_t} \left[  \Indicator\{\overline{\mathcal{E}_t} \}  \sum_\Action \Policy_{\Params_t}(\Action) \mE_{\Params^\star}[ \lvert(\Params_t- \Params^\star) \cdot \Context_{t,\Action} \rvert\mid\History_t,\Params_t]\bigg\vert\History_t \right] \\
 \leq&  \mE_{\Params_t}\left[ \mE_{\Action\sim \Policy_{\Params_t}}  \left( (1 + \sqrt{2\ln(\NumActions\TimeSteps/\delta')}) \sqrt{\Context_{t,\Action}^{\top} \CovMatrix_t^{-1} \Context_{t,\Action}} \right) \bigg\vert \History_t  \right]    +\underbrace{ \mE_{\Params_t} \left[  \Indicator\{\overline{\mathcal{E}_t} \}  \sum_\Action \Policy_{\Params_t}(\Action) \mE_{\Params^\star}[ \lvert(\Params_t- \Params^\star) \cdot \Context_{t,\Action} \rvert\mid\History_t,\Params_t]\bigg\vert\History_t \right] }_{\text{term b}}. 
\end{align*}
Below we bound term $b$ above. 
First note that: 

\begin{align*}
\mE_{\Params^\star}\left[ \lvert (\Params_t - \Params^\star) \cdot \Context_{t,\Action}\rvert^2\vert \History_t,\Params_t\right] &  \leq 2 ((\Params_t-\hat{\Params}_t)\cdot \Context_{t,\Action})^2 +2 \mE_{\Params^\star}\left[ ((\Params^\star-\hat{\Params}_t)\cdot \Context_{t,\Action})^2\vert\History_t,\Params_t\right] \\
& = 2 ((\Params_t - \hat{\Params}_t)\cdot \Context_{t,\Action})^2 + 2 \Context_{t,\Action}^{\top} \CovMatrix_t^{-1} \Context_{t,\Action} \leq 2 ((\Params_t - \hat{\Params}_t)\cdot \Context_{t,\Action})^2 + 2, 
\end{align*}

where the first inequality uses the fact that $(a+b)^2 \leq2a^2 + 2b^2$, the first equality uses the fact that $(\Params^\star -\hat{\Params}_t)\cdot \Context_{t,\Action}\sim \mathcal{N}(0, \Context_{t,\Action}^{\top} \CovMatrix_t^{-1} \Context_{t,\Action})$, and in
the last inequality we use $\|\Context_{t,\Action}\|_2 \leq 1$ and $\det(\CovMatrix_t^{-1}) \leq 1$.

For term b above, we now can upper bound it as:

\begin{align*}
\text{term b}&\leq \mE_{\Params_t} \left[ \Indicator\{\overline{\mathcal{E}_t}\}  \sqrt{\sum_\Action \Policy_{\Params_t}(\Action)^2} \sqrt{ \sum_\Action  \mE_{\Params^\star} \left[ \left((\Params_t-\Params^\star)\cdot \Context_{t,\Action}\right)^2\mid\History_t,\Params_t\right] } \bigg\vert\History_t\right]\\
&\leq  \mE_{\Params_t} \left[ \Indicator\{\overline{\mathcal{E}_t}\}  \sqrt{ \sum_\Action  \mE_{\Params^\star} \left[ \left((\Params_t-\Params^\star)\cdot \Context_{t,\Action}\right)^2\mid\History_t,\Params_t\right] } \bigg\vert\History_t\right]\\
& \leq 2\mE_{\Params_t} \left[ \Indicator\{\overline{\mathcal{E}_t}\}   \sqrt{ \sum_\Action  ( (\Params_t-\hat{\Params}_t)\cdot \Context_{t,\Action} )^2 + \NumActions  } \bigg\vert\History_t \right] \\
&\leq  2\mE_{\Params_t} \left[ \Indicator\{\overline{\mathcal{E}_t}\}   \sqrt{ \sum_\Action  ( (\Params_t-\hat{\Params}_t)\cdot \Context_{t,\Action} )^2 }   \bigg\vert\History_t \right] + 2\mE_{\Params_t} \left[ \Indicator\{\overline{\mathcal{E}_t}\}   \sqrt{  \NumActions  }       \right].
\end{align*}
 Note that we can further upper bound the first term on the RHS of the above inequality as follows:
\begin{align*}
 \mE_{\Params_t} \left[ \Indicator\{\overline{\mathcal{E}_t}\}   \sqrt{ \sum_\Action  ( (\Params_t-\hat{\Params}_t)\cdot \Context_{t,\Action} )^2 } \bigg\vert\History_t \right]  \leq \sqrt{ \mE_{\Params_t}\left[ \Indicator\{ \Params_t\in\overline{\mathcal{E}_t} \}\mid\History_t\right]  }  \sqrt{\mE_{\Params_t}  \left[\sum_\Action ( (\Params_t-\hat{\Params}_t )\cdot \Context_{t,\Action})^2\bigg\vert\History_t  \right]     } ,  
\end{align*} where we use the inequality that $\mE[uv] \leq \sqrt{\mE[u^2]} \sqrt{\mE[ v^2 ]} $.
Also note that:
\begin{align*}
\mE_{\Params_t} \left[   \sum_\Action ((\Params_t -\hat{\Params}_t)\cdot \Context_{t,\Action})^2 \bigg\vert\History_t \right] = \sum_\Action \Context_{t,\Action}^{\top} \CovMatrix_t^{-1} \Context_{t,\Action} \leq \NumActions,
\end{align*} since $(\Params_t - \hat{\Params}_t)\cdot \Context_{t,\Action} \sim \Normal\left( 0,  \Context_{t,\Action}^{\top}\CovMatrix_t^{-1} \Context_{t,\Action}\right)$.
Hence, we have:
\begin{align*}
 \mE_{\Params_t} \left[ \Indicator\{\overline{\mathcal{E}_t}\}   \sqrt{ \sum_\Action  ( (\Params_t-\hat{\Params}_t)\cdot \Context_{t,\Action} )^2 } \bigg\vert\History_t \right]  \leq \sqrt{ \mE_{\Params_t}\left[ \Indicator\{\Params_t \in\overline{\mathcal{E}_t}\}\mid\History_t\right]} \sqrt{ \NumActions}.
\end{align*}
This implies that for term $b$, we have:
\begin{align*}
\text{term b} \leq 2\sqrt{ \mE_{\Params_t}\left[ \Indicator\{\Params_t \in\overline{\mathcal{E}_t}\}\mid\History_t\right]} \sqrt{ \NumActions} + 2\mE_{\Params_t} \left[ \Indicator\{\overline{\mathcal{E}_t}\} \sqrt{\NumActions} \right]  = 2\left( \sqrt{ \mathbb{P}(\overline{\mathcal{E}_t}\mid\History_t)}   +  \mathbb{P}(\overline{\mathcal{E}_t}\mid\History_t) \right) \sqrt{\NumActions}.
\end{align*}

Sum over $\TimeSteps$ episodes, we have:
\begin{align*}
&\mE\left[ \sum_{t=1}^\TimeSteps   \mE_{\Params_t} \left[ \sum_\Action \Policy_{\Params_t}(\Action) \mE_{\Params^\star} \left[\lvert(\Params_t- \Params^\star) \cdot \Context_{t,\Action} \rvert\mid\History_t,\Params_t\right]\bigg\vert\History_t \right]  \right] \\
\leq& \frac{4\LipConst}{\MeritMin} \sum_{t=1}^\TimeSteps \mE_{\History_t}\left[ \mE_{\Params_t}\left[ \mE_{\Action\sim \Policy_{\Params_t}}  \left( (1 + \sqrt{2\ln(\NumActions\TimeSteps/\delta')}) \sqrt{\Context_{t,\Action}^{\top} \CovMatrix_t^{-1} \Context_{t,\Action}} \right) \bigg\vert \History_t  \right]   + \left(\sqrt{ \mathbb{P}(\overline{\mathcal{E}}_t\mid\History_t)} + \mathbb{P}(\overline{\mathcal{E}_t}\mid\History_t) \right)\sqrt{\NumActions}   \right] \\
=&\frac{4\LipConst}{\MeritMin}(1 + \sqrt{2\ln(\NumActions\TimeSteps/\delta')})  \mE\left[ \sum_{t=1}^\TimeSteps \mE_{\Action\sim \Policy_{\Params_t}} \sqrt{\Context_{t,\Action}^{\top} \CovMatrix_t^{-1} \Context_{t,\Action}}     \right] +  \frac{4\LipConst}{\MeritMin}\sqrt{\NumActions}\left( \sum_{t=1}^\TimeSteps \mE_{\History_t}\left[ \sqrt{\mathbb{P}(\overline{\mathcal{E}_t}\mid\History_t)}  + \mathbb{P}(\overline{\mathcal{E}_t}\mid\History_t) \right]    \right). 
\end{align*} For $\sum_{t=1}^\TimeSteps \mE_{\History_t}\left[ \sqrt{\mathbb{P}(\overline{\mathcal{E}_t}\mid\History_t)}  + \mathbb{P}(\overline{\mathcal{E}_t}\mid\History_t)\right]  $, we have:
\begin{align*}
\sum_{t=1}^\TimeSteps \mE_{\History_t}\left[ \sqrt{\mathbb{P}(\overline{\mathcal{E}_t}\mid\History_t)}  + \mathbb{P}(\overline{\mathcal{E}_t}\mid\History_t) \right]   & = \sum_{t=1}^\TimeSteps \mE_{\History_t} \sqrt{\mathbb{P}(\overline{\mathcal{E}_t}\mid\History_t)} +  \sum_{t=1}^\TimeSteps \mathbb{P}(\overline{\mathcal{E}_t})  \leq \sum_{t=1}^\TimeSteps \mE_{\History_t} \sqrt{\mathbb{P}(\overline{\mathcal{E}_t}\mid\History_t)} + \delta' \\
& \leq \sqrt{\TimeSteps} \sqrt{ \sum_{t=1}^\TimeSteps  \mE_{\History_t} \mathbb{P}(\overline{\mathcal{E}_t}\mid\History_t) } + \delta' \leq \sqrt{\TimeSteps \delta' }+ \delta'.
\end{align*}
Hence, we have:
\begin{align*}
&\mE\left[ \sum_{t=1}^\TimeSteps   \mE_{\Params_t} \left[ \sum_\Action \Policy_{\Params_t}(\Action) \mE_{\Params^\star} \left[\lvert(\Params_t- \Params^\star) \cdot \Context_{t,\Action} \rvert\mid\History_t,\Params_t\right]\bigg\vert\History_t \right]  \right]  \\
& \leq \frac{4\LipConst}{\MeritMin} (1 + \sqrt{2\ln(\NumActions\TimeSteps/\delta')})  \mE\left[ \sum_{t=1}^\TimeSteps \mE_{\Action\sim \Policy_{\Params_t}} \sqrt{\Context_{t,\Action}^{\top} \CovMatrix_t^{-1} \Context_{t,\Action}}     \right] +  \frac{4\LipConst}{\MeritMin}\sqrt{\NumActions}\left( \sqrt{\TimeSteps\delta'}+\delta'\right). 
\end{align*}

From Lemma~\ref{lemma:fair_LinUCB_concentration_width}, we have for any $\delta\in(0,1)$,

\begin{align*}
\mE\left[ \sum_{t=1}^\TimeSteps \mE_{\Action\sim \Policy_{\Params_t}} \sqrt{\Context_{t,\Action}^{\top} \CovMatrix_t^{-1} \Context_{t,\Action}}     \right]\leq& \delta\TimeSteps + \sqrt{2\TimeSteps\ln(4/\delta)} + \sum_{t=1}^\TimeSteps\Width_{t,\Action_t}\\
\leq & \delta\TimeSteps + \sqrt{2\TimeSteps\ln(4/\delta)} + \sqrt{2\TimeSteps\Dimension\ln(1+\frac{\TimeSteps}{\Dimension}) }. 
\end{align*}

Let $\delta = 1/\TimeSteps$, we have 
\[
\mE\left[ \sum_{t=1}^\TimeSteps \mE_{\Action\sim \Policy_{\Params_t}} \sqrt{\Context_{t,\Action}^{\top} \CovMatrix_t^{-1} \Context_{t,\Action}}     \right]\leq 1 + \sqrt{2\TimeSteps\ln(4\TimeSteps)} + \sqrt{2\TimeSteps\Dimension\ln(1+\frac{\TimeSteps}{\Dimension}) }. 
\]

Hence, we have:
\begin{align*}
&\mE\left[ \sum_{t=1}^\TimeSteps   \mE_{\Params_t} \left[ \sum_\Action \Policy_{\Params_t}(\Action) \mE_{\Params^\star} \left[\lvert(\Params_t- \Params^\star) \cdot \Context_{t,\Action} \rvert\mid\History_t,\Params_t\right]\bigg\vert\History_t \right]  \right]  \\
& \leq  \frac{4\LipConst}{\MeritMin}(1 + \sqrt{2\ln(\NumActions\TimeSteps/\delta')}) \left(\sqrt{2\TimeSteps\Dimension\ln(1+\frac{\TimeSteps}{\Dimension})} + \sqrt{2\TimeSteps\ln(4\TimeSteps) }+ 1 \right) + \frac{4\LipConst}{\MeritMin}\sqrt{\NumActions} (\sqrt{T \delta'} + \delta'). 
\end{align*}

Set $\delta' =  1 / (\NumActions \TimeSteps)$, we have:
\begin{align*}
\mE\left[ \sum_{t=1}^\TimeSteps   \mE_{\Params_t} \left[ \sum_\Action \Policy_{\Params_t}(\Action) \mE_{\Params^\star} \left[\lvert(\Params_t- \Params^\star) \cdot \Context_{t,\Action} \rvert\mid\History_t,\Params_t\right]\bigg\vert\History_t \right]  \right] =\widetilde{O}( \LipConst\sqrt{\TimeSteps\Dimension}/\MeritMin ). 
\end{align*} 
\end{proof}

\subsubsection{Proof of Theorem~\ref{theo:fair_LinTS_RR}}
\begin{lemma}[Adapted from Proposition 1 from~\cite{russo2014learning}]
\label{lemma:fair_TS_RRD}
For any UCB sequence $\left(\UCB_t:t\in\mathbbm{N}\right)$, the Bayesian reward regret of \FairXLinTS\ can be decomposed as follows:
\[
\text{Bayes}\RewardRegret_\TimeSteps = \mE\sum_{t=1}^\TimeSteps\left[ \mE_{\Action\sim\Policy_t}\left[\UCB_{t,\Action}-\Params^{\star}\cdot\Context_{t,\Action}\right] \right] + \mE\sum_{t=1}^\TimeSteps\left[ \mE_{\Action\sim\Policy^{\star}_t}\left[\Params^{\star}\cdot\Context_{t,\Action}-\UCB_{t,\Action}\right] \right]. 
\]
\end{lemma}

\begin{proof}
Note that at any round $t$, conditioned on history $\History_t$, the optimal fair policy $\Policy^{\star}_t$ and the deployed policy $\Policy_t$ selected by posterior sampling are identically distributed. In addition, $\UCB_t$ is deterministic and fully determined by the history $\History_t$. Hence $\mE\left[\mE_{\Action\sim\Policy_t}\left[\UCB_{t,\Action}\right]\mid\History_t\right]= \mE\left[\mE_{\Action\sim\Policy_t^{\star}}\left[\UCB_{t,\Action}\right]\mid\History_t\right]$. Therefore
\begin{align*}
&\mE[ \mE_{\Action\sim\Policy^{\star}_t}[\Params^{\star}\cdot\Context_{t,\Action}] -\mE_{\Action\sim\Policy_t}[\Params^{\star}\cdot\Context_{t,\Action}] ]\\
=&\mE_{\History_t}[ \mE[ \mE_{\Action\sim\Policy^{\star}_t}[\Params^{\star}\cdot\Context_{t,\Action}] -\mE_{\Action\sim\Policy_t}[\Params^{\star}\cdot\Context_{t,\Action}]  |\History_t]]\\
=&\mE_{\History_t}[ \mE[ \mE_{\Action\sim\Policy^{\star}_t}[\Params^{\star}\cdot\Context_{t,\Action}] -\mE_{\Action\sim\Policy^\star_t}\UCB_{t,\Action}+ \mE_{\Action\sim\Policy_t}\UCB_{t,\Action} - \mE_{\Action\sim\Policy_t}[\Params^{\star}\cdot\Context_{t,\Action}]  |\History_t]]\\
=&\mE[ \mE_{\Action\sim\Policy_t}[ \UCB_{t,\Action} - \Params^{\star}\cdot\Context_{t,\Action} ] ] + \mE[ \mE_{\Action\sim\Policy^{\star}}[\Params^{\star}\cdot\Context_{t,\Action}-\UCB_{t,\Action}] ]. 
\end{align*}

Summing over $\TimeSteps$ steps concludes the proof.
\end{proof}
Indeed, the above lemma holds for any $U_t$ that is fully determined by the history $\History_t$ which does not have to be a valid upper bound.

\begin{proof} (proof of Theorem~\ref{theo:fair_LinTS_RR})

Since Lemma~\ref{lemma:fair_TS_RRD} holds for any confidence sequences, we construct confidence sequences from the confidence analysis of the \FairXLinUCB\ algorithm in Proposition~\ref{prop:confidence} for this proof. The upper bounds of an arm $\Action$ across rounds are constructed as $\UCB_{t,\Action}\coloneqq\max_{\Params\in\ConfidenceRegion_t}\Params\cdot\Context_{t,\Action}$ and the lower bounds are constructed as $\LCB_{t,\Action}\coloneqq\min_{\Params\in\ConfidenceRegion_t}\Params\cdot\Context_{t,\Action}$. 

We bound the reward regret under different conditions on the following three events. The first event is that the norm of the true parameter $\Params^\star$ sampled from the prior normal distribution is not too large. We denote the event that $\|\Params^\star\|_2\leq W$ as $\mathcal{E}_1$. Since each dimension of  $\Params^\star$ is independently sampled from the standard normal distribution, $\mathbb{P}(\overline{\mathcal{E}}_1)\leq2\Dimension \exp(-W^2/2\Dimension)$. The second event is that for all rounds, the true parameter is within the confidence region of the \FairXLinUCB\ algorithm, which is the event in Proposition~\ref{prop:confidence}. We denote the event $\forall t$, $\Params^\star\in\ConfidenceRegion_t$ as $\mathcal{E}_2$. By proposition~\ref{prop:confidence}, we know that $\mathbb{P}(\overline{\mathcal{E}}_2)\leq \delta/2$. The third event is the event in Lemma~\ref{lemma:fair_LinUCB_concentration_width}. We denote $\left\lvert\sum_{t=1}^{\TimeSteps}\Width_{t,\Action_t} - \sum_{t=1}^\TimeSteps\mE_{\Action\sim\Policy_t}\Width_{t,\Action}\right\rvert\leq \sqrt{ 2\TimeSteps\ln(4/\delta)}$ as $\mathcal{E}_3$ and $\mathbb{P}(\overline{\mathcal{E}}_3)\leq \delta/2$.

By Lemma~\ref{lemma:fair_TS_RRD}, we have that

\begin{align*}
&\text{Bayes}\RewardRegret_\TimeSteps\\
=&\mE\left[\sum_{t=1}^\TimeSteps\sum_\Action(\Policy^\star_t(\Action)-\Policy_t(\Action))\Params^\star\cdot\Context_{t,\Action}\right]\\
=&\mE\left[\Indicator\{\mathcal{E}_2\text{ and } \mathcal{E}_3\}\sum_{t=1}^\TimeSteps\sum_\Action(\Policy^\star_t(\Action)-\Policy_t(\Action))\Params^\star\cdot\Context_{t,\Action}\right] + \underbrace{\mE\left[\Indicator\{\overline{\mathcal{E}}_2\text{ or }\overline{\mathcal{E}}_3\}\sum_{t=1}^\TimeSteps\sum_\Action(\Policy^\star_t(\Action)-\Policy_t(\Action))\Params^\star\cdot\Context_{t,\Action}\right]}_{\text{term c}} \\
= & \mE\left[\Indicator\{\mathcal{E}_2\text{ and }\mathcal{E}_3\} \sum_{t=1}^\TimeSteps\mE_{\Action\sim\Policy_t}[ \UCB_{t,\Action} - \Params^{\star}\cdot\Context_{t,\Action} ] \right] + \mE\left[\Indicator\{\mathcal{E}_2\text{ and }\mathcal{E}_3\} \sum_{t=1}^\TimeSteps\mE_{\Action\sim\Policy^{\star}}[\Params^{\star}\cdot\Context_{t,\Action}-\UCB_{t,\Action}] \right] + \text{ term c}\\
\leq & \mE\left[\Indicator\{\mathcal{E}_2\text{ and }\mathcal{E}_3\} \sum_{t=1}^\TimeSteps\mE_{\Action\sim\Policy_t}[ \UCB_{t,\Action} - \Params^{\star}\cdot\Context_{t,\Action} ] \right]  + \text{ term c}\\
=& \underbrace{\mE\left[\Indicator\{\mathcal{E}_1\text{ and }\mathcal{E}_2\text{ and }\mathcal{E}_3\} \sum_{t=1}^\TimeSteps\mE_{\Action\sim\Policy_t}[ \UCB_{t,\Action} - \Params^{\star}\cdot\Context_{t,\Action} ] \right]}_{\text{term a}} + \underbrace{ \mE\left[\Indicator\{\overline{\mathcal{E}}_1\text{ and }\mathcal{E}_2 \text{ and }\mathcal{E}_3\} \sum_{t=1}^\TimeSteps\mE_{\Action\sim\Policy_t}[ \UCB_{t,\Action} - \Params^{\star}\cdot\Context_{t,\Action} ] \right]}_{\text{term b}}  + \text{ term c}. 
\end{align*}

The inequality holds because under $\mathcal{E}_2$, $\Params^\star\cdot\Context_{t,\Action}\leq\UCB_{t,\Action}$. We now bound the three terms as follows.  
\begin{align*}
\text{term a}& =\mE\left[\Indicator\{\mathcal{E}_1\text{ and }\mathcal{E}_2\text{ and }\mathcal{E}_3\} \sum_{t=1}^\TimeSteps\mE_{\Action\sim\Policy_t}[ \UCB_{t,\Action} - \Params^{\star}\cdot\Context_{t,\Action} ] \right]\\
&\leq\mE\left[\Indicator\{\mathcal{E}_1\text{ and }\mathcal{E}_2\text{ and }\mathcal{E}_3\} \sum_{t=1}^\TimeSteps\mE_{\Action\sim\Policy_t}[ \UCB_{t,\Action} -\LCB_{t,\Action} ]\right]\\
&\leq 2\left(W + \sqrt{ \Dimension\ln(1+\TimeSteps/\Dimension)+2\ln(\TimeSteps^2\pi^2/3\delta) }\right)\left(\sqrt{2\TimeSteps\Dimension\ln(1+\frac{\TimeSteps}{\Dimension})}+ 2\sqrt{2\TimeSteps\ln(4/\delta)}\right). 
\end{align*}

The first inequality holds because under $\mathcal{E}_2$, $\LCB_{t,\Action}$ is a lower bound of $\Params^\star\Context_{t,\Action}$.  The derivation of the second inequality is the same as the proof of the reward regret of the \FairXLinUCB\ algorithm. 

To bound $\text{term c}$, we upper bound on the expected instantaneous reward regret under the event $\mathcal{E}_1$ and the event $\overline{\mathcal{E}}_1$ respectively. Under event $\mathcal{E}_1$, we have that at each round $t$, the expected instantaneous reward regret
\begin{align*}
\mE\left[\Indicator\{\mathcal{E}_1\}\sum_\Action(\Policy^\star_t(\Action)-\Policy_t(\Action))\Params^\star\cdot\Context_{t,\Action}\right]
&\leq \mE\left[\Indicator\{\|\Params^\star\|_2\leq W\}\sum_\Action\left\lvert\Policy^\star_t(\Action)-\Policy_t(\Action)\right\rvert\left\lvert\Params^\star\cdot\Context_{t,\Action}\right\rvert\right]\\
&\leq\mE\left[\Indicator\{\|\Params^\star\|_2\leq W\}\sum_\Action\left\lvert\Policy^\star_t(\Action)-\Policy_t(\Action)\right\rvert\|\Params^\star\|_2\cdot\|\Context_{t,\Action}\|_2\right]\\
&\leq 2W.
\end{align*}

Under event $\overline{\mathcal{E}}_1$, we have that at each round $t$, the instantaneous reward regret

\begin{align*}
\mE\left[\Indicator\{\overline{\mathcal{E}}_1\}\sum_\Action(\Policy^\star_t(\Action)-\Policy_t(\Action))\Params^\star\cdot\Context_{t,\Action}\right] &\leq \mE\left[\Indicator\{\|\Params^\star\|_2> W\}\sum_\Action\left\lvert\Policy^\star_t(\Action)-\Policy_t(\Action)\right\rvert\|\Params^\star\|_2\cdot\|\Context_{t,\Action}\|_2\right] \\
&\leq 2\mE\left[\Indicator\{ \|\Params^\star\|_2 > W \}\|\Params^\star\|_2\right]\\
&=2\mE_{\Params^\star \sim \Normal(0, I)}\left[ \Indicator\{ \|\Params^\star\|_2 > W \} \|\Params^\star \|_{2}\right]\\  & \leq  2\sqrt{ \mE_{\Params^\star \sim \Normal(0,I)} \Indicator\{  \| \Params^\star \|_2 \geq W  \} }   \sqrt{ \mE_{\Params^\star\sim \Normal(0, I)}  \|\Params^\star  \|_2^2   }\\
& \leq  2\sqrt{ 2 \Dimension \exp\left(  -W^2 / 2\Dimension \right)} \sqrt{  \mE_{\Params^\star\sim \Normal(0, I)} \left[ \sum_{i,j} \Params^\star_i \mu^\star_j   \right]   } \\
& \leq  2\sqrt{ 2 \Dimension \exp\left(  -W^2 / 2\Dimension \right)} \sqrt{ \Dimension }\\
&=2\Dimension\sqrt{ 2 \exp\left(  -W^2 / 2\Dimension \right)}.
\end{align*} Here the third inequality uses $\mE[ab] \leq \sqrt{\mE[a^2]}\sqrt{\mE[b^2]}$. Thus

\begin{align*}
\text{term c} &= \mE\left[\Indicator\{\overline{\mathcal{E}}_2\text{ or }\overline{\mathcal{E}}_3\}\sum_{t=1}^\TimeSteps\sum_\Action(\Policy^\star_t(\Action)-\Policy_t(\Action))\Params^\star\cdot\Context_{t,\Action}\right]\\
&\leq  \delta\TimeSteps(2\Dimension\sqrt{ 2 \exp\left(  -W^2 / 2\Dimension \right)} +2W), 
\end{align*}

since no matter $\mathcal{E}_1$ or $\overline{\mathcal{E}}_1$, the expected instantaneous reward regret is less than the sum of the two situations.

For \text{term b}, 
denote $\Params^\UCB_{t,\Action}$ the parameter that achieves the upper confidence bound $\UCB_{t,\Action}$ in the confidence region $\ConfidenceRegion_t$. 

\begin{align*}
\text{term b}&= \mE\left[\Indicator\{\overline{\mathcal{E}}_1\text{ and }\mathcal{E}_2 \text{ and }\mathcal{E}_3\} \sum_{t=1}^\TimeSteps\mE_{\Action\sim\Policy_t}[ \UCB_{t,\Action} - \Params^{\star}\cdot\Context_{t,\Action} ] \right]\\
&= \mE\left[\Indicator\{\overline{\mathcal{E}}_1\text{ and }\mathcal{E}_2 \text{ and }\mathcal{E}_3\} \sum_{t=1}^\TimeSteps\mE_{\Action\sim\Policy_t}[ \Params^\UCB_{t,\Action}\cdot\Context_{t,\Action} - \Params^{\star}\cdot\Context_{t,\Action} ] \right]\\
&= \mE\left[\Indicator\{\overline{\mathcal{E}}_1\text{ and }\mathcal{E}_2 \text{ and }\mathcal{E}_3\} \sum_{t=1}^\TimeSteps\mE_{\Action\sim\Policy_t}[ \Params^\UCB_{t,\Action}\cdot\Context_{t,\Action} -\hat{\Params}_{t}\cdot\Context_{t,\Action} + \hat{\Params}_{t}\cdot\Context_{t,\Action}  - \Params^{\star}\cdot\Context_{t,\Action} ] \right]\\
&\leq \mE\left[\Indicator\{\overline{\mathcal{E}}_1\text{ and }\mathcal{E}_2 \text{ and }\mathcal{E}_3\} \sum_{t=1}^\TimeSteps\mE_{\Action\sim\Policy_t}\left[ \|\Params^\UCB_{t,\Action} - \hat{\Params}_t\|_{\CovMatrix_t}\cdot\|\Context_{t,\Action}\|_{\CovMatrix_t^{-1}} + \|\hat{\Params}_{t} - \Params^\star\|_{\CovMatrix_t}\cdot\|\Context_{t,\Action}\|_{\CovMatrix_t^{-1}} \right] \right]\\
&\leq 2\TimeSteps\mE_{\Params^\star\sim\Normal(0,I)}\left[ \Indicator\{\overline{\mathcal{E}}_1\text{ and }\mathcal{E}_2 \text{ and }\mathcal{E}_3\}(\sqrt{ \Dimension\ln(1+\TimeSteps/\Dimension)+2\ln(\TimeSteps^2\pi^2/3\delta) } + \|\Params^\star\|_2) \right] \\
&\leq 4\TimeSteps\Dimension\exp(-W^2/2\Dimension) \sqrt{ \Dimension\ln(1+\TimeSteps/\Dimension)+2\ln(\TimeSteps^2\pi^2/3\delta) } + 2\TimeSteps\Dimension\sqrt{2\exp(-W^2/2\Dimension)}. 
\end{align*}

The second inequality holds because $\|\Context_{t,\Action}\|_{\CovMatrix^{-1}}\leq1$ and by the results of the confidence analysis in Section~\ref{sec:confidence_analysis}. The last inequality comes from the analysis of \text{term c}.

So
\begin{align*}
\text{Bayes}\RewardRegret_\TimeSteps\leq& 2\left(W + \sqrt{ \Dimension\ln(1+\TimeSteps/\Dimension)+2\ln(\TimeSteps^2\pi^2/3\delta) }\right)\left(\sqrt{2\TimeSteps\Dimension\ln(1+\frac{\TimeSteps}{\Dimension})}+ 2\sqrt{2\TimeSteps\ln(4/\delta)}\right)\\
&+ 4\TimeSteps\Dimension\exp(-W^2/2\Dimension) \sqrt{ \Dimension\ln(1+\TimeSteps/\Dimension)+2\ln(\TimeSteps^2\pi^2/3\delta) } + 2\TimeSteps\Dimension\sqrt{2\exp(-W^2/2\Dimension)}\\
&+  \delta\TimeSteps(2\Dimension\sqrt{ 2 \exp\left(  -W^2 / 2\Dimension \right)} +2W) 
\end{align*}

Let $W = \sqrt{2\Dimension\ln(\TimeSteps\Dimension)}$ and $\delta = 1/\TimeSteps$, we have that 
\[
\text{Bayes}\RewardRegret_\TimeSteps=\widetilde{O}\left(\Dimension\sqrt{\TimeSteps}\right)
\]

\end{proof}

\subsection{Proof of Theorem~\ref{theo:fair_TS_FR}}
\begin{proof}
We can convert a stochastic MAB instance into a stochastic linear bandit instance by constructing $\NumActions$ $\NumActions$-dimensional basis vectors, each representing an arm. Then the upper bounds derived for linear bandits also hold for MAB. The $\widetilde{O}\left(\LipConst\sqrt{\NumActions\TimeSteps}/\MeritMin\right)$ fairness regret upper bound follows. 
\end{proof}

\subsection{Proof of Theorem~\ref{theo:fair_TS_RR}}
\begin{proof}
We can use the confidence sequence in the \FairXUCB\ algorithm and apply Lemma~\ref{lemma:fair_TS_RRD} to get the $\widetilde{O}\left( \sqrt{\NumActions\TimeSteps} \right)$ reward regret upper bound similarly as in the proof of Theorem~\ref{theo:fair_LinTS_RR}. 
\end{proof}

\section{Additional Experiments}
\label{appendix:exp}
In this section, we present additional experiment results to illustrate the effectiveness of different algorithms across datasets and merit functions.

\subsection{Additional Experiment Setup}

We conducted experiments on simulation data from the mediamill dataset~\cite{snoek2006challenge} in addtion to the yeast dataset. The mediamill dataset consists of $43,907$ examples. Each example has $120$ features and belongs to one or multiple of the $101$ classes. Similar to preparing the yeast dataset, we randomly split the dataset into two sets, $20\%$ as the validation set to tune hyper-parameters and $80\%$ as the test set to test the performance of different algorithms. All the details of the experiments are the same as the experiments on the yeast dataset as introduced in Section~\ref{sec:exp_setup}. 

The ranges of the hyper-parameters we searched for each bandit algorithm on the two simulation datasets are as follows. For \FairXUCB\ and UCB, we grid search $\ConfiWidth\in[1e-5, 2e-5, 5e-5,1e-4, 2e-4, 5e-4,1e-3, 2e-3, 5e-3, 1e-2, 2e-2, 5e-2, 1e-1]$. For \FairXTS and TS, we grid search the normal prior standard deviation in $[1e-5, 1e-4, 1e-3, 1e-2,1e-1,1]$ and reward standard deviation in $[1e-5, 1e-4, 1e-3, 1e-2,1e-1,1]$. For \FairXEG, we grid search $\epsilon\in[0, 1e-4, 1e-3, 1e-2]$. For \FairXLinUCB\ and LinUCB, we grid search $\beta\in[0.01, 0.1, 1.0, 10.0]$ and $\lambda\in[1, 1e1, 1e2, 1e3, 1e4, 1e5]$. For \FairXLinTS\ and LinTS, we grid search prior standard deviation in $[1e-6,1e-5,1e-4,1e-3,1e-2,1e-1]$ and reward standard deviation in $[1e-6, 1e-5, 1e-4, 1e-3, 1e-2, 1e-1]$. For \FairXEG, we grid search $\epsilon\in[0,1e-4,1e-3,1e-2]$ and the regularization parameter of the ridge regression in $[1e1, 1e2,1e3,1e4,1e5]$. 

For the experiments on the Yahoo! dataset, in addition to the results where we select hyper-paramters on the logs in the first day and test the performance on the second day, we here present the results where we select hyper-parameters on the first $5$ days and test the performance on the rest $5$ days. The whole Yahoo! dataset contains $45,811,883$ events and $1,633,488$ clicks on $271$ articles. 

The ranges of the hyper-parameters we searched for each bandit algorithm on the Yahoo! dataset are as follows. For \FairXUCB\ and UCB, we grid search $\ConfiWidth\in[1e-5, 1e-4, 1e-3, 1e-2, 1e-1]$. For \FairXTS\ and TS, we grid search the normal prior standard deviation in $[0.01, 0.1]$ and reward standard deviation in $[1e-5, 1e-4, 1e-3, 1e-2, 1e-1]$. For \FairXEG, we grid search $\epsilon\in[0, 1e-3, 1e-2]$. For \FairXLinUCB\ and LinUCB, we grid search $\beta\in[1, 10, 100]$ and $\lambda\in[1e-6, 1e-4, 1e-2, 1]$. For \FairXLinTS\ and LinTS, we grid search prior standard deviation in $[1e-4,1e-3,1e-2]$ and reward standard deviation in $[1e-4,1e-3,1e-2]$. For \FairXLinEG, we grid search $\epsilon\in[0,1e-4,1e-3]$ and the regularization parameter of the ridge regression in $[1e-6, 1e-4,1e-2,1]$. 
\begin{figure*}[!tbh]
\begin{subfigure}{.245\textwidth}
  \centering
  \includegraphics[width=\linewidth]{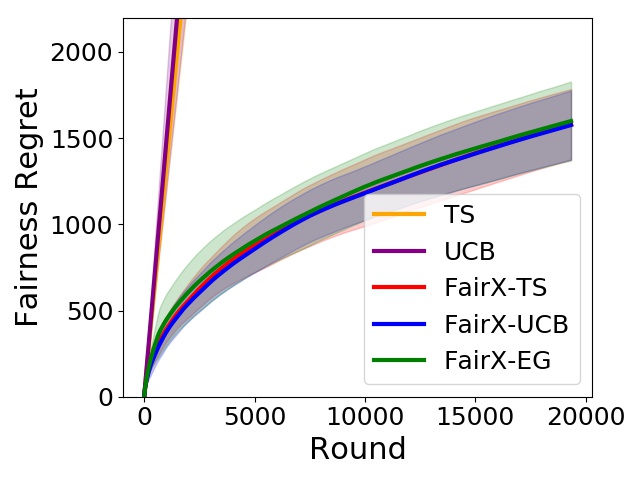}
\end{subfigure}
\begin{subfigure}{.245\textwidth}
  \centering
  \includegraphics[width=\linewidth]{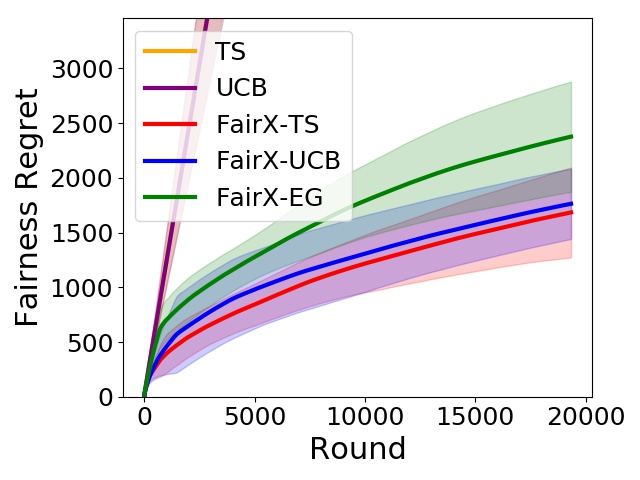}
\end{subfigure}
\begin{subfigure}{.245\textwidth}
  \centering
  \includegraphics[width=\linewidth]{figures/yeast_L10_fr.jpeg}
\end{subfigure}
\begin{subfigure}{.245\textwidth}
  \centering
  \includegraphics[width=\linewidth]{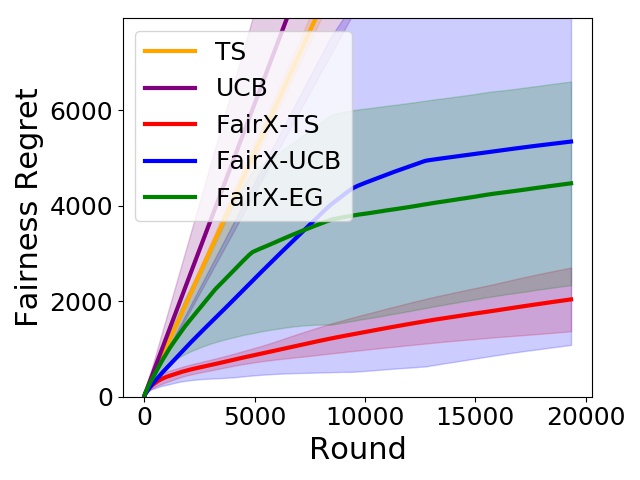}
\end{subfigure}
\begin{subfigure}{.245\textwidth}
  \centering
  \includegraphics[width=\linewidth]{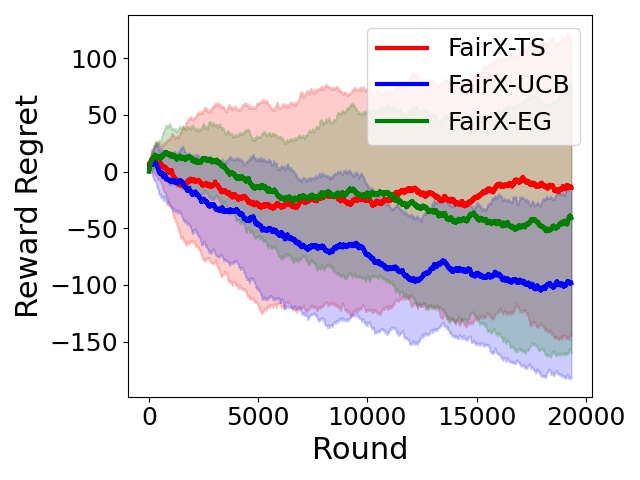}
  \caption{$c = 6$}
\end{subfigure}
\begin{subfigure}{.245\textwidth}
  \centering
  \includegraphics[width=\linewidth]{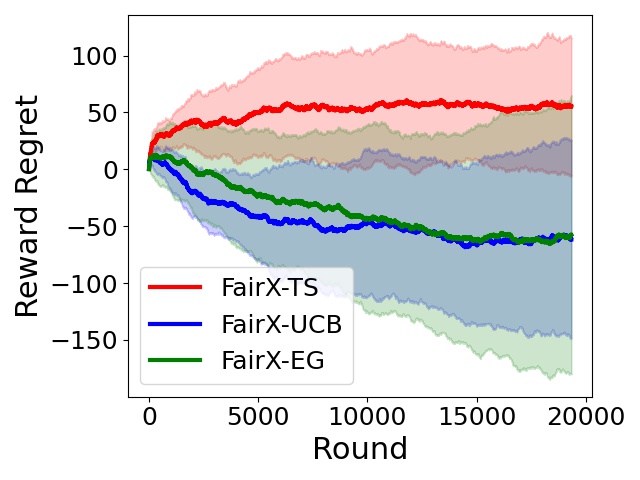}
  \caption{$c = 8$}
\end{subfigure}
\begin{subfigure}{.245\textwidth}
  \centering
  \includegraphics[width=\linewidth]{figures/yeast_L10_rr.jpeg}
  \caption{$c = 10$}
\end{subfigure}
\begin{subfigure}{.245\textwidth}
  \centering
  \includegraphics[width=\linewidth]{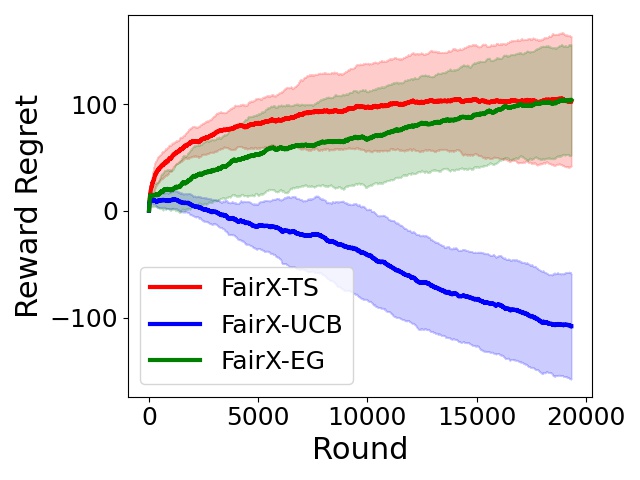}
  \caption{$c = 12$}
\end{subfigure}
\caption{Experiment results on the yeast dataset with varying merit function parameter $c$ for different MAB algorithms.}
\label{fig:yeast_all}
\end{figure*}

\begin{figure*}[!tbh]
\begin{subfigure}{.245\textwidth}
  \centering
  \includegraphics[width=\linewidth]{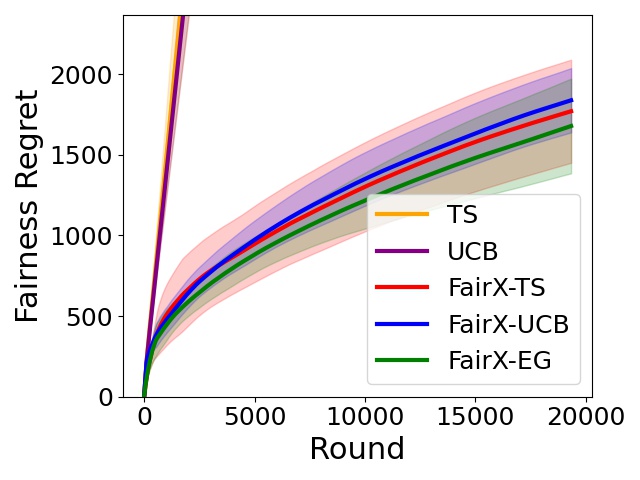}
\end{subfigure}
\begin{subfigure}{.245\textwidth}
  \centering
  \includegraphics[width=\linewidth]{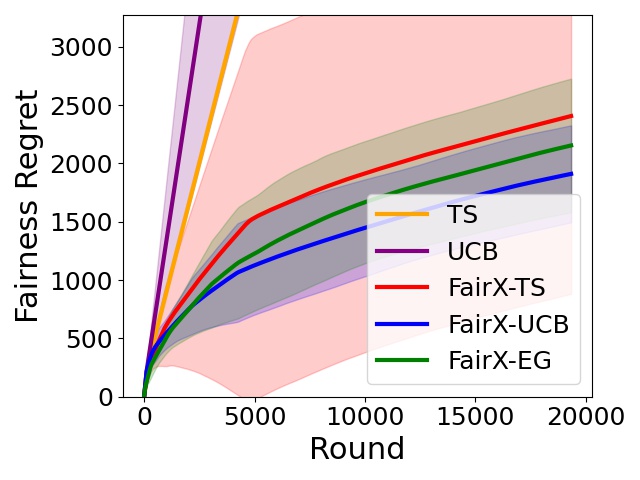}
\end{subfigure}
\begin{subfigure}{.245\textwidth}
  \centering
  \includegraphics[width=\linewidth]{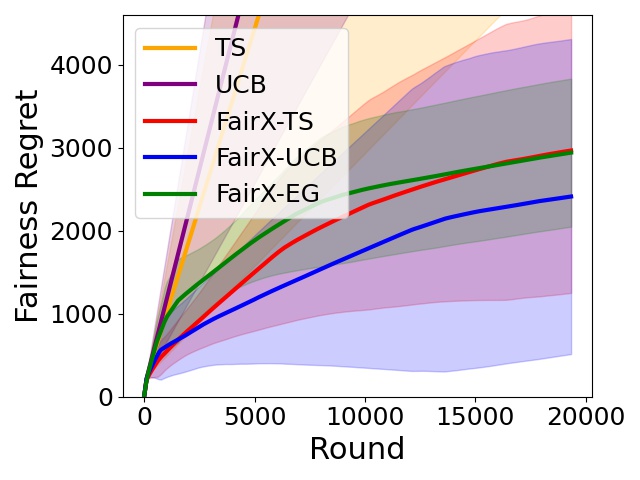}
\end{subfigure}
\begin{subfigure}{.245\textwidth}
  \centering
  \includegraphics[width=\linewidth]{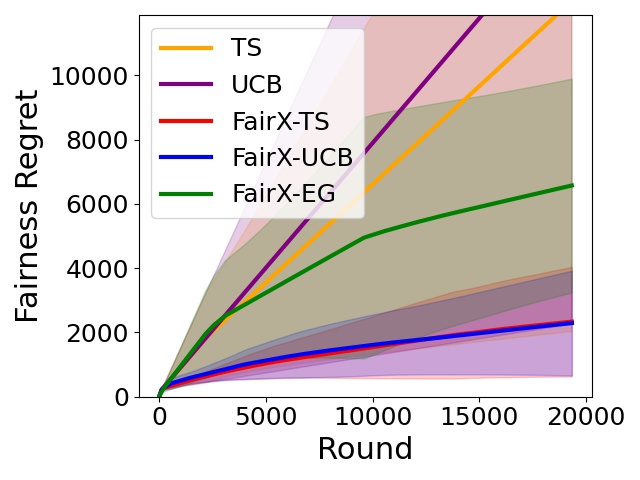}
\end{subfigure}
\begin{subfigure}{.245\textwidth}
  \centering
  \includegraphics[width=\linewidth]{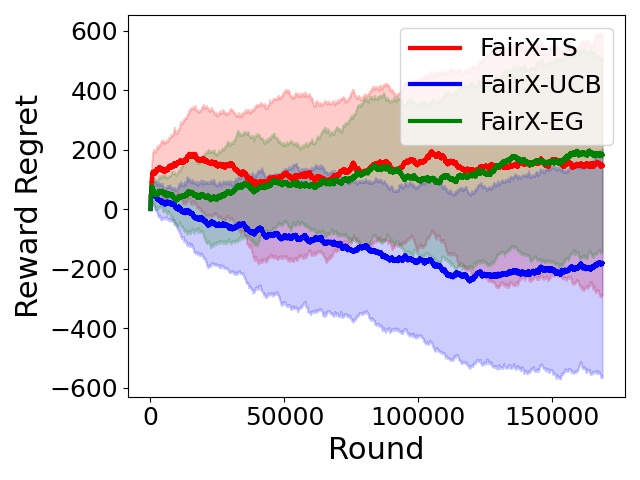}
  \caption{$c = 6$}
\end{subfigure}
\begin{subfigure}{.245\textwidth}
  \centering
  \includegraphics[width=\linewidth]{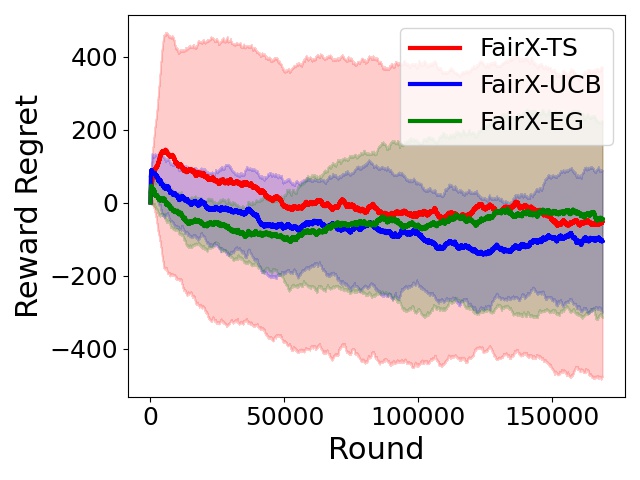}
  \caption{$c = 8$}
\end{subfigure}
\begin{subfigure}{.245\textwidth}
  \centering
  \includegraphics[width=\linewidth]{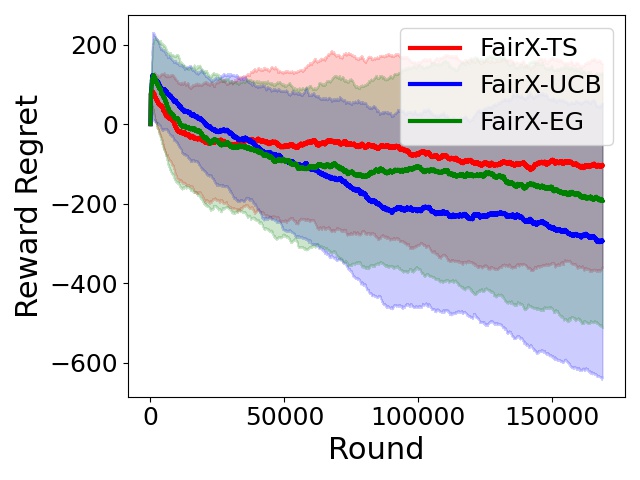}
  \caption{$c = 10$}
\end{subfigure}
\begin{subfigure}{.245\textwidth}
  \centering
  \includegraphics[width=\linewidth]{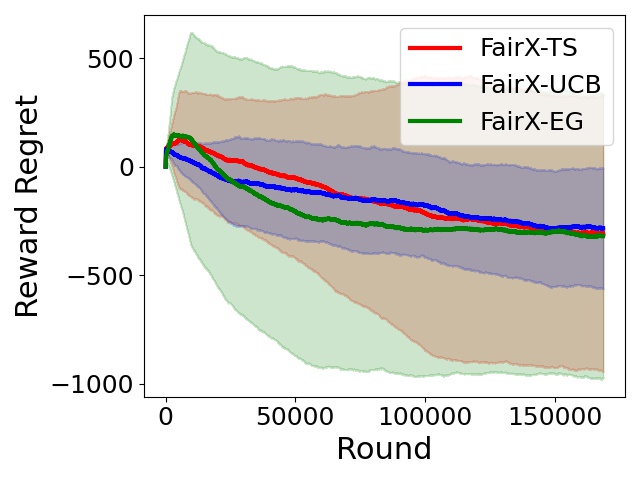}
  \caption{$c = 12$}
\end{subfigure}
\caption{Experiment results on the mediamill dataset with varying merit function parameter $c$ for different MAB algorithms.}
\label{fig:mediamill_all}
\end{figure*}

\begin{figure*}[!tbh]
\begin{subfigure}{.33\textwidth}
  \centering
  \includegraphics[width=\linewidth]{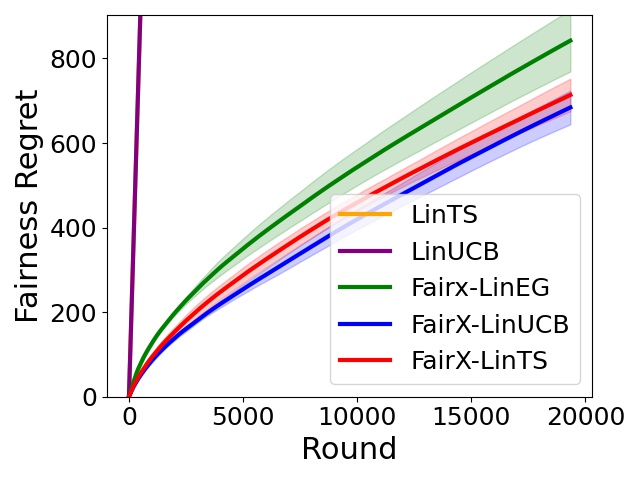}
\end{subfigure}
\begin{subfigure}{.33\textwidth}
  \centering
  \includegraphics[width=\linewidth]{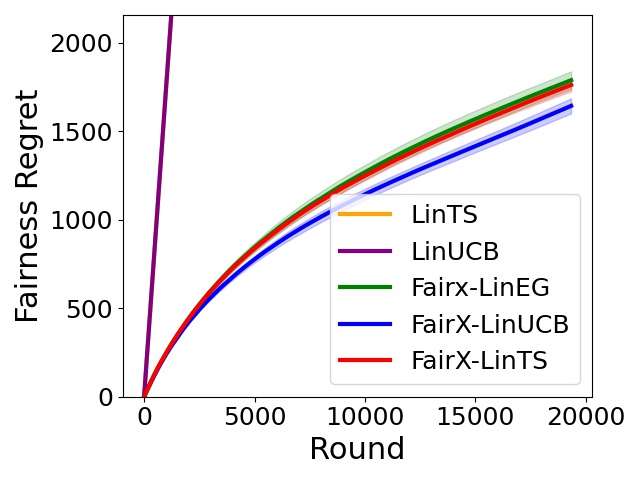}
\end{subfigure}
\begin{subfigure}{.33\textwidth}
  \centering
  \includegraphics[width=\linewidth]{figures/yeast_lin_L3_fr.jpeg}
\end{subfigure}
\begin{subfigure}{.33\textwidth}
  \centering
  \includegraphics[width=\linewidth]{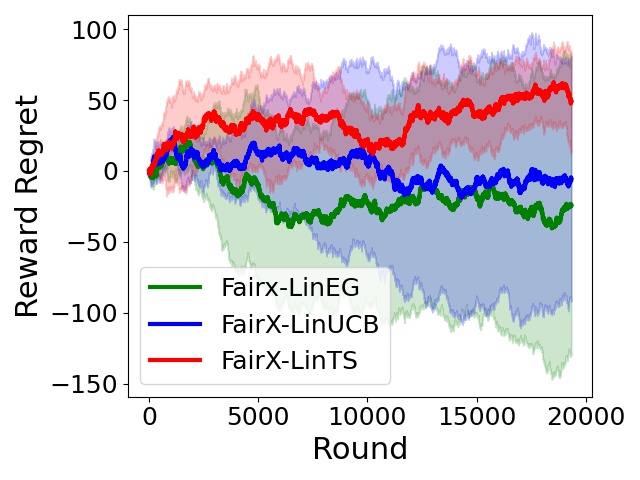}
  \caption{$c = 1$}
\end{subfigure}
\begin{subfigure}{.33\textwidth}
  \centering
  \includegraphics[width=\linewidth]{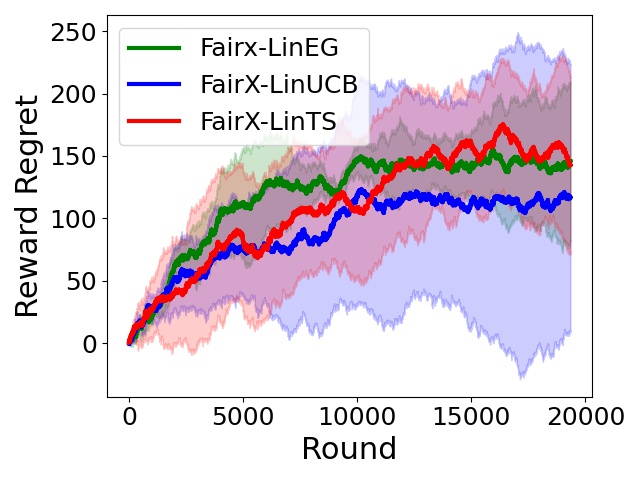}
  \caption{$c = 2$}
\end{subfigure}
\begin{subfigure}{.33\textwidth}
  \centering
  \includegraphics[width=\linewidth]{figures/yeast_lin_L3_rr.jpeg}
  \caption{$c = 3$}
\end{subfigure}
\caption{Experiment results on the yeast dataset with varying merit function parameter $c$ for different linear bandit algorithms.}
\label{fig:yeast_lin_all}
\end{figure*}

\begin{figure*}[!tbh]
\begin{subfigure}{.33\textwidth}
  \centering
  \includegraphics[width=\linewidth]{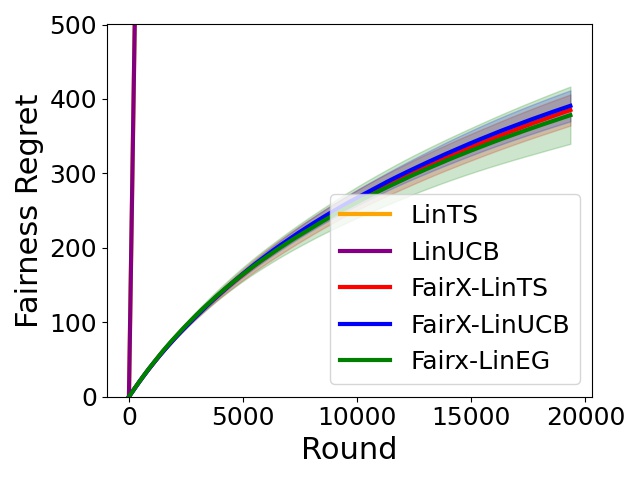}
\end{subfigure}
\begin{subfigure}{.33\textwidth}
  \centering
  \includegraphics[width=\linewidth]{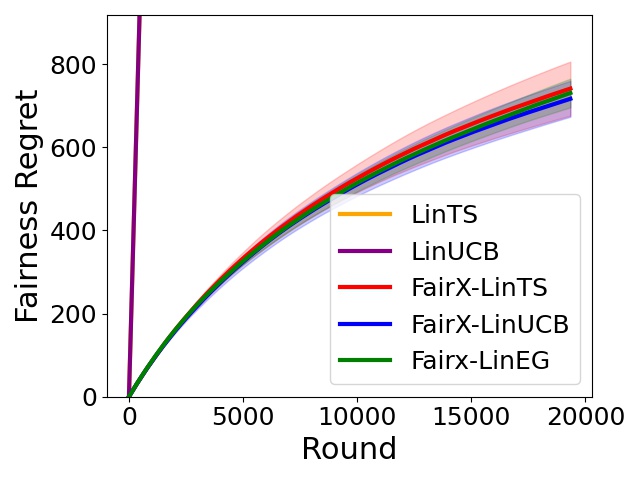}
\end{subfigure}
\begin{subfigure}{.33\textwidth}
  \centering
  \includegraphics[width=\linewidth]{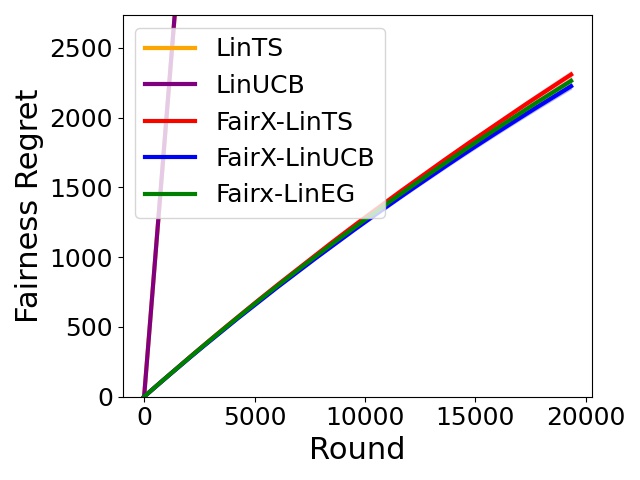}
\end{subfigure}
\begin{subfigure}{.33\textwidth}
  \centering
  \includegraphics[width=\linewidth]{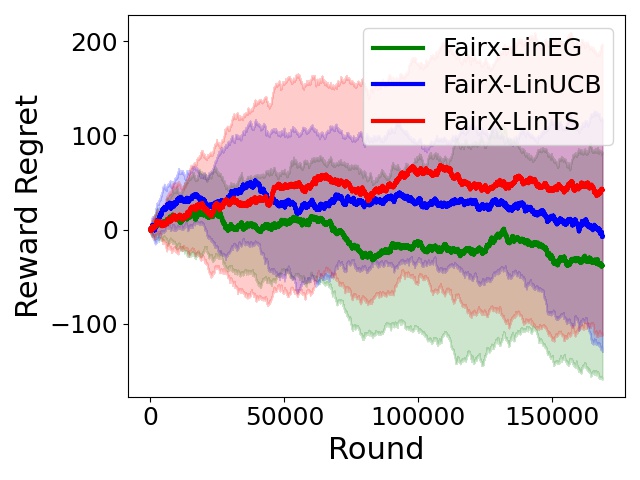}
  \caption{$c = 1$}
\end{subfigure}
\begin{subfigure}{.33\textwidth}
  \centering
  \includegraphics[width=\linewidth]{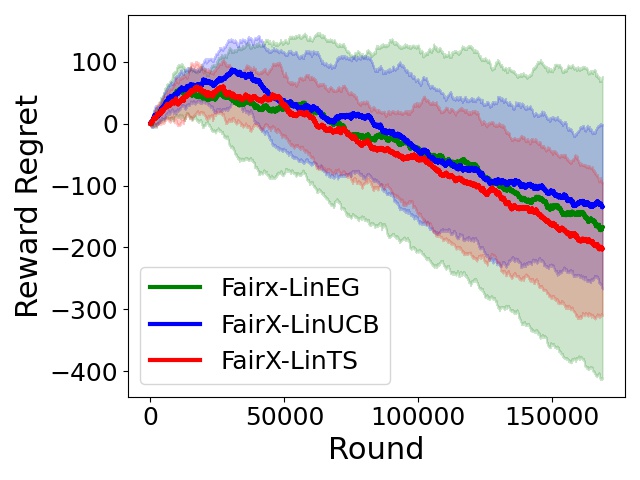}
  \caption{$c = 2$}
\end{subfigure}
\begin{subfigure}{.33\textwidth}
  \centering
  \includegraphics[width=\linewidth]{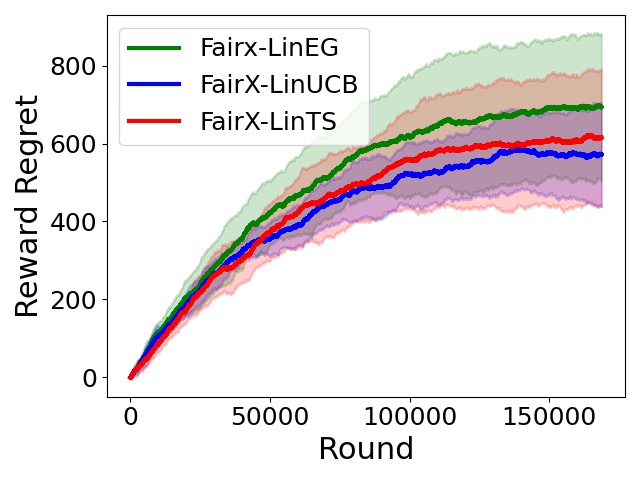}
  \caption{$c = 3$}
\end{subfigure}
\caption{Experiment results on the mediamill dataset with varying merit function parameter $c$ for different linear bandit algorithms.}
\label{fig:mediamill_lin_all}
\end{figure*}

\begin{figure*}[!tbh]
\begin{subfigure}{.245\textwidth}
  \centering
  \includegraphics[width=\linewidth]{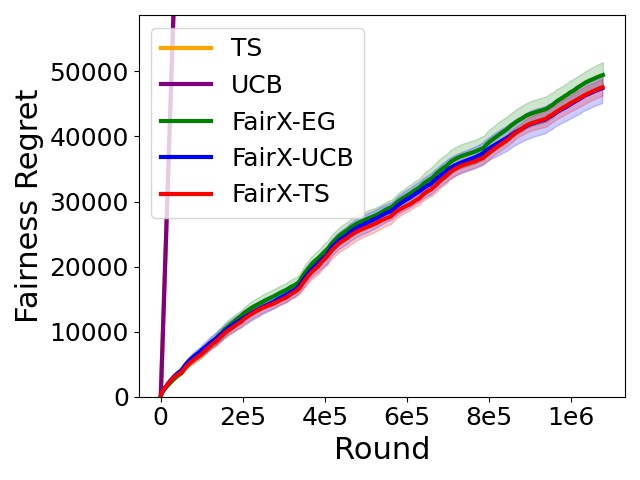}
\end{subfigure}
\begin{subfigure}{.245\textwidth}
  \centering
  \includegraphics[width=\linewidth]{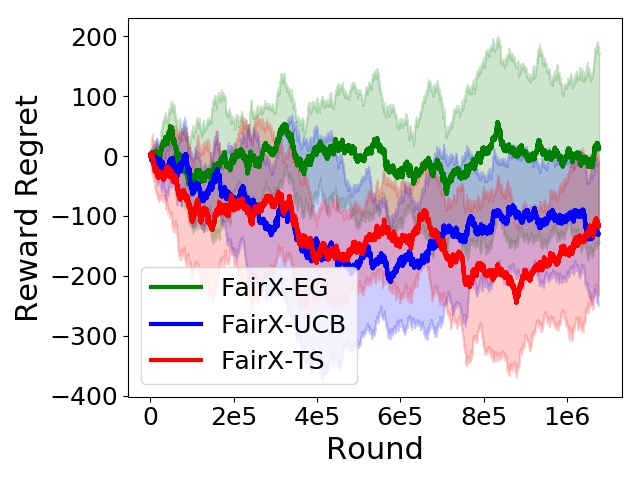}
\end{subfigure}
\begin{subfigure}{.245\textwidth}
  \centering
  \includegraphics[width=\linewidth]{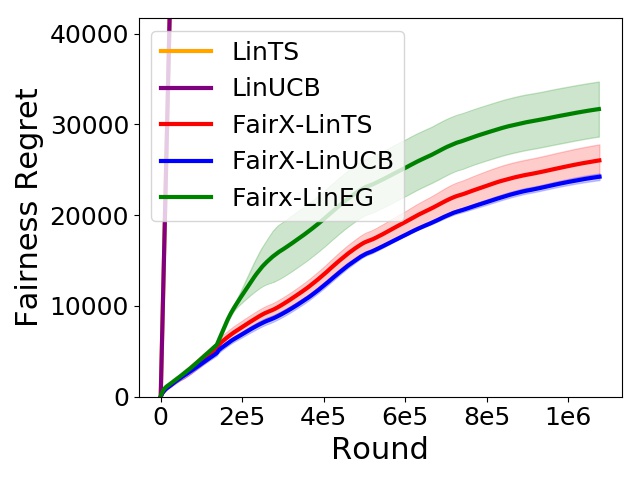}
\end{subfigure}
\begin{subfigure}{.245\textwidth}
  \centering
  \includegraphics[width=\linewidth]{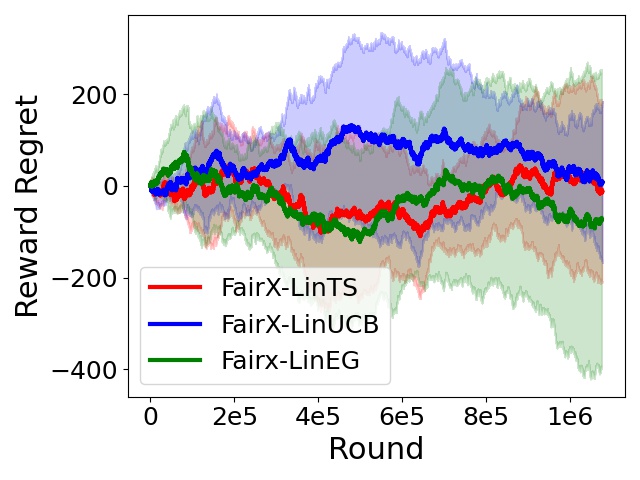}
\end{subfigure}
\caption{Experiment results on the Yahoo! dataset (select hyper-parameters on data from the first $5$ days and report the performance on the data from the last $5$ days) for both MAB and linear bandits setting. ($c=10$ for both settings) }
\label{fig:yahoo_all}
\end{figure*}

\subsection{Additional Experiment Results}

We show the experiment results with varying merit function parameter $c$ in the \FairX\ MAB setting and \FairX\ linear bandits setting on the yeast and mediamill datasets in \autoref{fig:yeast_all}, \autoref{fig:mediamill_all}, \autoref{fig:yeast_lin_all}, and \autoref{fig:mediamill_lin_all}. All the \FairX\ algorithms can effectively control merit-based fairness of exposure while achieving low reward regret. As the merit function becomes steeper, the variance of the runs and the difference between algorithms become larger.

The experiment results on the Yahoo! dataset are shown in \autoref{fig:yahoo_all}. Though the regrets occasionally jump a bit (which might be due to the change of candidate articles or user interests), all the \FairX\ algorithms can robustly control merit-based fairness of exposure while maintaining low reward regret. 

\begin{figure*}[!tbh]
\begin{subfigure}{.45\textwidth}
  \centering
  \includegraphics[width=\linewidth]{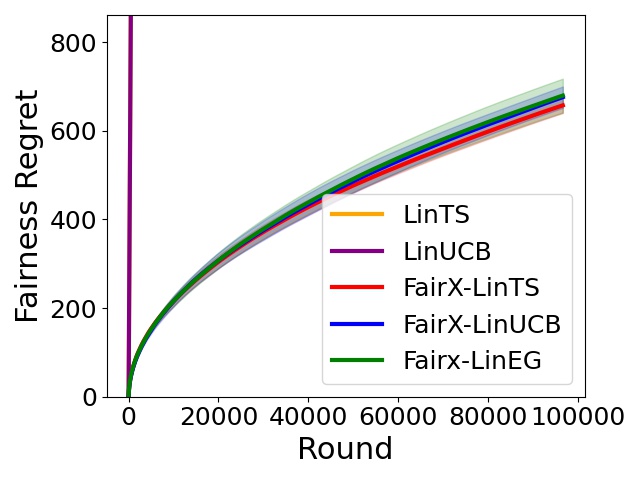}
\end{subfigure}
\begin{subfigure}{.45\textwidth}
  \centering
  \includegraphics[width=\linewidth]{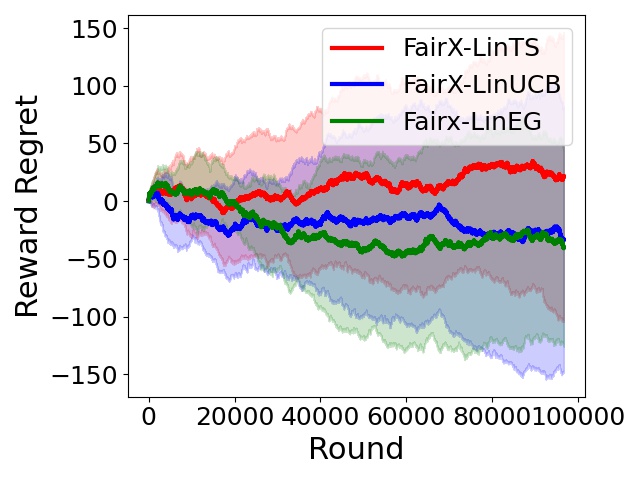}
\end{subfigure}
\caption{Experiment results on the yeast dataset with a well-specified linear model for linear bandit algorithms. ($c = 2$)}
\label{fig:yeast_lin_specified}
\end{figure*}

\subsection{Experiment for Linear Bandits with a Well-Specified Linear Model}
As discussed in the main paper, the fairness regret of \FairX\ linear bandit algorithms do not seem to converge because the linearity assumption does not necessarily hold for any of the datasets. In this section, we will create a dataset with a well-specified linear model and see how these bandit algorithms perform on this data. We perform the experiments on the yeast dataset. To remove model misspecification, we do not use the class labels as rewards, but instead inpute the linear least-squares solution on the full-information data as rewards (plus Gaussian noise $\Normal(0,0,1))$. The results are shown in \autoref{fig:yeast_lin_specified}. We can see that the fairness regret of the \FairX\ algorithms converge better and follow the $\sqrt{\TimeSteps}$ bound predicted by the theoretical analysis in the absence of model mis-specification.

\end{document}